\documentclass[11pt]{article}

\usepackage{amsfonts,amssymb,mathrsfs,amsmath,amssymb,amsthm,float}
\usepackage{graphicx}
\usepackage{geometry}
\usepackage{algorithm,algorithmic}
\usepackage{color}
\usepackage{multirow}

\newtheorem{theorem}{Theorem}[section]
\newtheorem{lemma}{Lemma}[section]
\newtheorem{cor}{Corollary}[section]

\newtheorem{define}{Definition}[section]
\newtheorem{remark}{Remark}[section]

\def\I{{\mathbb{I}}}

\def\R{{\mathbb{R}}}
\def\FB{{\mathbb{F}}}

\def\N{{\mathbb{N}}}

\def\F{{\mathcal{F}}}

\def\U{{\mathcal{U}}}

\def\H{{\mathcal{H}}}

\def\Relu{{\hbox{\rm{Relu}}}}

\def\sign{{\hbox{\rm{sign}}}}
\def\diag{{\hbox{\rm{diag}}}}
\def\Proj{{\hbox{\rm{Proj}}}}

\def\CE{{\hbox{\scriptsize\rm{CE}}}}
\def\AT{{\hbox{\scriptsize\rm{AT}}}}

\def\d{{\hbox{\rm{d}}}}

\begin{document}


\title{Adversarial Parameter Attack on Deep Neural Networks\thanks{This work is partially supported by NSFC grant No.11688101 and  NKRDP grant No.2018YFA0704705.}}
\author{Lijia Yu, Yihan Wang, and Xiao-Shan Gao\\
 Academy of Mathematics and Systems Science,  Chinese Academy of Sciences, \\ Beijing 100190,  China\\
 University of  Chinese Academy of Sciences,  Beijing 100049,  China}
\date{\today}
\maketitle

\begin{abstract}
\noindent
In this paper, a new parameter perturbation attack  on DNNs, called adversarial parameter attack,  is proposed, in which small perturbations to the parameters of the DNN
are made such that the accuracy of the attacked DNN does not decrease much,
but its robustness becomes much lower.
The adversarial parameter attack is stronger than previous parameter perturbation attacks
in that the attack is more difficult to be recognized by users
and the attacked DNN gives a wrong label for any modified sample input with high probability.
The existence of adversarial parameters is proved.
For a DNN $\F_{\Theta}$ with the parameter set $\Theta$ satisfying certain conditions,
it is shown that if the depth of the DNN is sufficiently large,
then there exists an adversarial parameter set $\Theta_a$ for $\Theta$ such that
the accuracy of $\F_{\Theta_a}$ is equal to that of $\F_{\Theta}$,
but the robustness measure of $\F_{\Theta_a}$ is smaller than any given bound.
An effective training algorithm is given to compute
adversarial parameters and numerical experiments are used to demonstrate that
the algorithms are effective to produce high quality adversarial parameters.

\vskip10pt\noindent
{\bf Keyword}.
Adversarial parameter attack, adversarial samples, robustness measurement,
adversarial accuracy, mathematical theory for safe DNN.
\end{abstract}

\section{Introduction}

The deep neural network (DNN) \cite{lecun2015deep} has become the most powerful machine learning method, which has been successfully applied in computer vision, natural language processing, and many other fields.
Safety is a key desired feature of DNNs, which was studied extensively~\cite{sur1,sur2,rubust-rev1}.
%

The most  widely studied safety issue for DNNs is the
adversarial sample attack~\cite{S2013}, that is, it is possible to intentionally make small modifications to a sample,  which are essentially imperceptible to the human eye,
but the DNN outputs a wrong label or even any label given by the adversary.
Existence of adversary samples makes the DNN vulnerable in safety-critical applications
and many effective methods were proposed to develop more robust DNNs against adversarial attacks~\cite{M2017,sur1,sur2,rubust-rev1}.
However, it was shown that adversaries samples are inevitable for current DNN models in certain sense~\cite{Bast1,asulay1,adv-inev1}.
%

More recently, the parameter perturbation attacks~\cite{pp1,pp2,pp3,aq1,aq2,aq3,stealth1,stealth2}
were studied and shown to be another serious safety treat to DNNs.
It was shown that by making small parameter perturbations,
the attacked DNN can give wrong or desired labels to
specified input samples and still give the correct labels to other samples~\cite{pp1,pp2,stealth1,stealth2}.

In this paper, the adversarial parameter attack  is proposed,
in which small perturbations to the parameters of a DNN are made
such that the attack to the DNN is essentially imperceptible to the user,
but the robustness of the DNN becomes much lower.
The adversarial parameter attack is stronger than previous parameter perturbation attacks
in that not only the accuracy but also the robustness of DNNs are considered.

\subsection{Contributions}

Let $\F_\Theta$ be a DNN with $\Theta$ as the parameter set.
A parameter perturbation $\Theta_a$  is called a set of {\em adversarial parameters} of $F_\Theta$ or $\Theta$,
if the following conditions are satisfied
1) $\Theta_a$ is a small modification of $\Theta$, for instance
$||\Theta_a-\Theta||_\infty\le \epsilon$ for a small positive number $\epsilon$;
2) the accuracy of $\F_{\Theta_a}$ over a distribution of samples
is almost the same as that of $\F_{\Theta}$;
3) $\F_{\Theta_a}$ is much less robust than $\F_{\Theta}$, that is,
$\F_{\Theta_a}$ has much more adversarial samples than $\F_{\Theta}$.
It is clear that conditions 1) and 2) are to make the attack difficult to be recognized by the users and condition 3) is to make the new DNN less safe.
The DNN obtained by the above attack is called  an {\em adversarial DNN},
which has high accuracy but low robustness.


The existence of adversarial parameters is proved under certain assumptions.
It is shown that if the depth of a trained DNN $\F_\Theta$ is sufficiently large,
then there exist adversarial parameters $\Theta_a$ such that
the accuracy of $\F_{\Theta_a}$ is equal to that of $\F_{\Theta}$,
but the robustness measure of $\F_{\Theta_a}$ is as small as possible
(refer to Corollaries \ref{cor-32} and \ref{cor-41}).
Since $\F_{\Theta}$ is a continuous function in $\Theta$,
if $\Theta_a$ is an adversarial parameter for $\Theta$
then there exists a small sphere $S_a$ with $\Theta_a$ as center such that
all parameters in $S_a$ are also adversarial parameters for $\Theta$.
%
These results imply that adversarial parameters are inevitable in certain sense,
similar to adversarial samples~\cite{Bast1,asulay1,adv-inev1}.

The existence of adversarial samples is usually demonstrated
with numerical experiments, besides a few cases to be mention in the next section.
As an application of adversarial parameters, we can construct DNNs
which are guaranteed to have adversarial samples.
For a trained DNN $\F_\Theta$ satisfying certain conditions,
it is shown that
there exist adversarial parameters $\Theta_a$ such that
the accuracy of $\F_{\Theta_a}$ is equal to that of $\F_{\Theta}$,
but $\F_{\Theta_a}$ has adversarial samples near a given normal sample
(refer to Theorem \ref{th-1}),
or the probability for $\F_{\Theta_a}$ to have adversarial samples over a distribution of samples is at least $1/2$  (refer to Theorem \ref{th-2}).
%

Finally, an effective training algorithm is given to compute
adversarial parameters and numerical experiments are used to demonstrate that
the algorithms are effective to produce high quality adversarial parameters
for the networks VGG19 and Resnet56 on the CIFAR-10 dataset.


\subsection{Related work}

There exist vast literatures on adversarial attacks, which
can be found in the survey papers~\cite{sur1,sur2,rubust-rev1}.
We will focus on those  which are closely related to our work.

{\bf Parameter perturbation attacks.}
Parameter perturbation attacks were given under different names such as
fault injection attack, fault sneaking attack, stealth attack, and weight corruption.
The fault injection attack~\cite{pp1} was first proposed by Liu et al,
where it was shown that parameter perturbations can be used
to misclassify one given input sample.
In \cite{pp3}, it was shown that laser injection techniques can be used as
a successful fault injection attack in real-world applications.
In \cite{pp2}, the fault sneaking attack was proposed,
where multiple input samples were misclassified and
other samples were still given the correct label.
%
In \cite{aq3}, lower bounds were given for parameter perturbations under which the network still gives the correct label for a given sample.
In \cite{aq1}, upper bounds were given for the changes of the pairwise class margin function
and the Rademacher complexity against parameter perturbations
and new loss functions were given to obtain more robust networks.
In \cite{aq2}, the maximum change of the loss function over given samples was used
as an indicator to measure the robustness of DNNs against parameter perturbations
and gradient decent methods were used to compute the indicator.
In \cite{stealth1,stealth2}, the stealth attack which can guaranteed
to make the attacked DNN gives a desired label for a sample outside of the validation set
and keep correct labels for samples in the validation set.
The stealth attack has the form $\F+\U$,
where $\F$ is the DNN to be attacked and $\U$ is a DNN with one hidden layer.

The adversarial parameter attack proposed in this paper
is stronger than previous parameter perturbation attacks
by in the following aspects.
First, by keeping the accuracy and eliminating the robustness,
the adversarial parameter attack is more difficult to be recognized,
because the attached DNN performs almost the same as the original DNN on the
test set.
Second, by reducing the robustness of the DNN,
the attacked DNN gives a wrong label for any modified input sample with high probability,
while previous parameter attacks usually misclassify certain given samples.
Finally, we prove the existence of adversarial parameters under reasonable assumptions.

{\bf Mathematical theories of adversarial samples.}
Existence of adversarial samples were usually demonstrated with numerical experiments, and mathematical theories were desired.
In \cite{Bast1}, it was proved that for DNNs with a fixed architecture,
there exist uncountable classification functions and distributions of samples
such that adversarial samples always exist for any successfully trained DNN
with the given architecture and the sample distribution.
In the stealth attack~\cite{stealth1,stealth2},
it was proved that there exist attached DNNs which  give a desired label for a sample outside of the validation set
by modifying the DNN.
In this paper, we show that by making small perturbations to the parameters of the DNN,
the DNN has adversarial samples with high probability.

Theories for certified robustness of DNNs were given in several aspects.
Let $x$ be a sample such that the DNN $\F$ gives the correct label.
Due to the continuity of the DNN function, a  sphere
with $x$ as center does not contain adversarial samples if its radius is sufficiently small, which is called
a {\em robust sphere} of $x$.
In \cite{rob-b3}, lower bounds for the robustness radius were computed and used to enhance
the robustness of the DNN.
In \cite{rob-b1}, for shallow networks, the upper bounds of the changes of the network
under sample input perturbations were given and use to obtain more robust DNNs.
In \cite{smooth1}, the random smoothing method was proposed and
lower bounds for the radius of the robust spheres  was given.
In \cite{netl2}, lower bounds for the average radius of robust spheres for a distribution of samples are given.
Universal lower bounds on stability in terms of the dimension of the domain of
the classification function were also given in~\cite{adv-inev1,stealth1}.
However, these bounds are usually inverse-exponentially
dependent on the depth of the DNN, which are very small
for deep networks in real world applications.
In \cite{neth1}, the information-theoretically safe bias classifier
was introduced by making the gradient of the DNN random.

{\bf Algorithms to train robust DNNs.}
Many effective methods were proposed to train more robust DNNs to defend adversarial samples~\cite{sur1,sur2,rubust-rev1,sur-adv}.
Methods to train DNNs which are more robust against parameter perturbation attacks
were also proposed~\cite{pp1,pp2}.
The adversarial training method proposed by Madry et al~\cite{M2017}
can reduce adversarial samples significantly,
where the value of the loss function of the worst adversary in a small neighborhood of
the training sample is minimized.
In this paper, the idea of adversarial training is used to
compute adversarial parameters.

\section{Adversarial parameters}
\label{sec-def}
In this section, we define the adversarial parameters and give
a measurement for the quality of the adversarial parameters.

\subsection{Adversarial parameters of DNNs}
\label{sec-def1}

Let us consider  a standard DNN.
Denote $\I=[0,  1]\subset\R$ and $[n] = \{1, \ldots, n\}$ for $n\in\N_{+}$.
In this paper,  we assume that $\F:\I^n\to \R^m$ is a
classification DNN for $m$ objects,
which has $L$ hidden-layers, all hidden-layers use $\Relu$ as the activity function,
and the output layer does not have activity functions. $\F$ can be written as
\begin{equation}
\label{eq-dnn}
\renewcommand{\arraystretch}{1.2}
\begin{array}{ll}
x_0\in\I^n,  n_0=n,  n_{L+1}=m;\\
x_{l}=\Relu(W_{l}x_{l-1}+b_{l}) \in \R^{n_{l}},  W_{l}\in \R^{n_l\times n_{l-1}},   b_{l}\in \R^{n_l},  l\in[L]; \\
 \F(x_0)=x_{L+1}=W_{L+1}x_{L}+b_{L+1},
 W_{L+1}\in \R^{m\times n_L},   b_{L+1}\in \R^{m}.\\
\end{array}
\end{equation}
Denote $\Theta=\{W_{l}, b_{l}\}_{l=1}^{L+1}\in\R^k$ to be the parameter set   of $\F$
and $\F$ is denoted as $\F_{\Theta}$ if the parameters need to be mentioned explicitly,
where $k=\sum_{l=1}^{L+1} n_l(n_{l-1}+1)$.

Let $\F_{\Theta}$ be a trained network with the parameter set $\Theta$.
Then a new parameter set $\Theta_a$ is called a set of {\em adversarial parameters} of $\Theta$
if
1) $\Theta_a$ is a small perturbation of $\Theta$;
2) the accuracy of $\F_{\Theta_a}$  is almost the same as that of $\F_{\Theta}$;
3) $\F_{\Theta_a}$ is much less robust comparing to $\F_{\Theta}$, that is,
$\F_{\Theta_a}$ has more adversarial samples than $\F_{\Theta}$.

We assume that the objects to be classified
satisfy a distribution $D_x$  in $\R^n$, and a sample $x\sim D_x$ is called a
{\em normal sample}.
Let $\Theta\in\R^k$ be the parameter set of a trained network $\F_{\Theta}:\I^n\to\R^m$.
For $x\sim D_x$, denote $l_x$ to be the label of $x$ and $\widehat{\F}_{\Theta}$ to be the classification result of $\F_{\Theta}$.
Then the accuracy of $\F_{\Theta}$ for the normal samples is
\begin{equation}
\label{eq-base}
A(\F_\Theta,D_x)=P_{x\sim D_x}(\widehat{\F}_{\Theta}(x)=l_x).
\end{equation}

In order to measure the quality of adversarial parameters, we need
a robustness measure $R(\F_{\Theta},D_x)$ of $\F_{\Theta}$ for the normal samples.
There exist several definitions for $R(\F_{\Theta},D_x)$~\cite{M2017,netl2}.
In this paper, two kinds of robustness measures are used.

We first give two robust measures of $\F_\Theta$ on a given sample $x_0$.
%
%
%
The {\em robustness radius} of $x_0$ under the $L_{p}$ norm for  $p\in\R_+\cup \{\infty\}$ is defined to be
\begin{equation}
\label{eq-rm11}
R_1(\F_{\Theta},x_0)=\max
\{\zeta\in\R_+\,|\,  \widehat{\F}(x)=l_x, \forall x \hbox{ s.t. } ||x-x_0||_p\le \zeta\}.
\end{equation}
If $\widehat{\F}_{\Theta}(x_0)\ne l_{x_0}$, then the robustness radius of $x_0$ is zero.
%
It is difficult to have good estimation to the robustness radius,
and the following approximation to the robust radius under $L_{p}$ norm~\cite{rob-b3} is often used
\begin{equation}
\label{eq-rm21}
R_2(\F_{\Theta},x_0)=\min_{l\in[m],l\ne l_x}\{\frac{|\F_{l_x}(x_0)-\F_{l}(x_0)|}{||\nabla \F_{l_x}(x_0)-\nabla \F_{l}(x_0)||_q}I(\F_{l_x}(x_0)>\F_{l}(x_0))\}
\end{equation}
where
$\F_l(x_0)$ is the $l$-th coordinate of $\F(x_0)$,
$\nabla(\F_l(x))=\frac{\nabla \F_l(t)}{\nabla t}|_{t=x}$,
$p,q\in\R_+\{\infty\}$ satisfy $1/q+1/p=1$ ($p=0(\infty)$ iff $q=\infty(0)$),
and $I(t)=1$ if $t$ it true or $I(t)=0$ otherwise.

For a distribution $D_x$ of  samples, we define two
global robust measures corresponding to $R_1$ and $R_2$.
The adversarial accuracy can be used as  $R(\F_{\Theta},D_x)$.
For   $\epsilon\in\R_+$ and $p\in\R_+\cup \{\infty\}$, the adversarial accuracy of $\F_{\Theta}$ is
\begin{equation}
\label{eq-rm12}
\begin{array}{ll}
R_3(\F_\Theta,D_x,\epsilon)&=P_{x\sim D_x}(\epsilon\le R_1(\F_\Theta,x))\\
&=P_{x\sim D_x}(\widehat{\F}_{\Theta}(x')=l_x, \forall x' \hbox{ s.t. } ||x'-x||_{p}\le \epsilon).
\end{array}
\end{equation}
Corresponding to $R_2$ in \eqref{eq-rm21}, we have the following global
robustness measure
\begin{equation}
\label{eq-rm22}
R_4(\F,D_x)=\int_{x\sim D_x}R_2(\F,x)\d x.
\end{equation}

We now define a measurement for an adversarial parameter set using the accuracy
and robustness of $\F$.
\begin{define}
\label{def-1}
Let $\Theta_{a}$ be an adversarial parameter set of $\Theta$,
$R(\F,D_x)$ a robustness measure of $\F$ for normal samples, and
\begin{eqnarray}
\label{eq-m1}
&&P_{x\sim D_x}({\F}_{\Theta_{a}}(x)=l_x)=\gamma_1P_{x\sim D_x}({\F}_{\Theta}(x)=l_x)\\ \nonumber
&&R(\F_{\Theta_{a}},D_x)= \gamma_2R(\F_{\Theta},D_x).
\end{eqnarray}
Then the {\em adversarial rate}   of $\Theta_{a}$ is defined to be $\overline{\gamma_1}(1-\overline{\gamma_2})$, where $\overline{\gamma}=\min\{\gamma,1\}$.
\end{define}

In general, we have $\gamma_1\le 1$ and $\gamma_2\le1$.
The value of $\gamma_1$ measures the ability of $\Theta_{a}$ to keep the accuracy
of $\F_\Theta$ on normal samples, and if $\gamma_1$ is large  then the attack is more difficult to be detected.
The value of $1-\gamma_2$ measures the ability of $\Theta_{a}$ to break the robustness of $\F_\Theta$, and if $1-\gamma_2$ is large  then the parameter attack is more powerful.
Hence, the adversarial rate  $\overline{\gamma_1}(1-\overline{\gamma_2})$
measures the quality of the adversarial parameter attack in that
if the adversarial rate  is larger then the adversarial parameter attack is better.
If $\overline{\gamma_1}(1-\overline{\gamma_2})$ achieves its maximal value $1$,
then  $\gamma_1=1$ and $\gamma_2=0$ and the
adversarial parameter attack is a perfect attack in that
the attack does not change the accuracy of $\F$,
but totally destroys the robustness of $\F$.

%

\begin{remark}
\label{rem-12}
In order to make the attack  very hard to be detected, we can give a lower bound $\gamma_{\rm{low}}$ to $\gamma_1$.
If $\gamma_1<\gamma_{\rm{low}}$, we consider $\Theta_{a}$ to be a failed attack.
\end{remark}

\subsection{Adversarial parameter attacks for other purposes}
\label{sec-def2}
According to the requirements of specific applications, we may define
other types of adversarial parameter attacks.

The adversarial parameters defined in section \ref{sec-def1} are for all the samples.
In certain applications, it is desired to make the network less robust on one specific
class of samples, which motivates the following definitions.

The simplest case is adversarial parameters for a given sample.
%
A small perturbation $\Theta_a$ of $\Theta$  is called {\em adversarial parameters}
for a given sample $x_0$, if $\F_{\Theta_a}$ gives the correct label to $x_0$
and   has adversarial samples of $x_0$ in
$S_{\infty}(x_0,\epsilon)=\{x\,|\, |x-x_0|_\infty\le \epsilon\}$
for a given $\epsilon\in\R_+$.
Let $R(\F_{\Theta},x_0)$ be a measure of robustness of $\F_{\Theta}$ at sample $x_0$, and
\begin{eqnarray}
\label{eq-m11}
&&R(\F_{\Theta_{a}},x_0)= \gamma R(\F_{\Theta},x_0).
\end{eqnarray}
Then the adversarial rate  of $\Theta_a$ is defined to be $1-\overline{\gamma}$.

We can also break the stability for samples with a given label.
A small perturbation $\Theta_a$ of $\Theta$  is called {\em adversarial parameters}
for samples with a given label $l_0$,
if $\F_{\Theta_a}$ keeps the accuracy for all normal samples and
the robustness for normal samples whose label is not $l_0$,
but break the robustness of samples with label $l_0$.
Let
$\alpha=P_{x\sim D_x}(\widehat{\F}_{\Theta}(x)=l_x)$
and
$\beta=P_{x\sim D_x}(\widehat{\F}_{\Theta}(x')= l_x,\forall x'\in S_{p}(x,\epsilon))$.
For such adversarial parameters $\Theta_a$, let
\begin{eqnarray}
\label{eq-z20}
&&P_{x\sim D_x}(\widehat{\F}_{\Theta_{a}}(x)=l_x)=\gamma_1\alpha\\ \nonumber
&&P_{x\sim D_x}(\widehat{\F}_{\Theta_{a}}(x')= l_x,\forall x'\in B_{p}(x,\epsilon)\,|\,l_x\ne y_0)
   = \gamma_2\beta\\
&&P_{x\sim D_x}(\widehat{\F}_{\Theta_{a}}(x')= l_x,\forall x'\in B_{p}(x,\epsilon)\,|\,l_x= y_0)
   = \gamma_3\beta.\nonumber
\end{eqnarray}
Then the  adversarial rate  of $\Theta_a$ is defined as $\overline{\gamma}_1\overline{\gamma}_2(1-\overline{\gamma}_3)$.

Finally, instead of breaking the robustness of samples with label $y_0$,
we can break the accuracies for samples  with label $y_0$.
Such adversarial parameters are called {\em direct adversarial parameters}.
Let $\Theta_{a}$ be a direct adversarial parameter set and
\begin{eqnarray}
\label{eq-z21}
&&P_{x\sim D_x}(\widehat{\F}_{\Theta_{a}}(x)=l_x\,|\,l_x\ne y_0)=\gamma_1\alpha\\ \nonumber
&&P_{x\sim D_x}(\widehat{\F}_{\Theta_{a}}(x')= l_x,\forall x'\in B_{p}(x,\epsilon)\,|\,l_x\ne y_0)
   = \gamma_2\beta\\
&& P_{x\sim D_x}(\widehat{\F}_{\Theta_{a}}(x)=l_x\,|\,l_x= y_0)=\gamma_3\alpha. \nonumber
\end{eqnarray}
Then the adversarial rate  is defined as $\overline{\gamma}_1\overline{\gamma}_2(1-\overline{\gamma}_3)$.
The above definition is similar to the attacks in \cite{pp1,pp2},
but robustness is considered as an extra objective.

\section{Algorithm}
\label{sec-alg}
In this section, we give algorithms to compute adversarial parameters.

\subsection{Compute adversarial parameters under $L_{\infty}$ norm}

We formulate the adversarial parameter attack for a trained DNN $\F_\Theta$
under the $L_p$ norm  as the following optimization problem
for a given $\zeta\in\R_+$.
\begin{eqnarray}
\label{eq-apa}
&&\max_{\Theta_a\in\R^k, ||\Theta_{a}-\Theta||_{p}\le\zeta}
A(\F_{\Theta_{a}},D_x)/R(\F_{\Theta_{a}},D_x)
\end{eqnarray}
where $A(\F_{\Theta_{a}},D_x)$ and $R(\F_{\Theta_{a}},D_x)$
are the accuracy and a robustness measure for $\F_{\Theta_{a}}$ over a distribution sample $D_x$.

\begin{remark}
\label{rem-31}
Theoretically, the adversarial parameter attack should be a multi-objective
optimization problem, that is to maximize the accuracy
and to minimize the robustness.
But, such an optimization problem is difficult to solve.
%
%
\end{remark}

\begin{remark}
\label{rem-32}
According to \eqref{eq-apa}, the adversarial rate seems better to be defined
as $\gamma_1/\gamma_2$, which is possible but not as good as the one in
Definition \ref{def-1} for the following reasons.
The adversarial rate  $\overline{\gamma_1}(1-\overline{\gamma_2})$
has the optimal value $1$ and gives a more intuitive view
to see the quality of the adversarial parameters.
\end{remark}

In the rest of this section, we show how to change formula \eqref{eq-apa}
to an effective algorithm to compute $L_{\infty}$ norm adversarial
parameters using the robustness measure in \eqref{eq-rm12}.
We first show how to compute the robustness
in \eqref{eq-rm12} explicitly.
We use the adversarial training~\cite{M2017} to do that, which is the
most effective way to find adversarial samples.
For a sample $x$ and a small number $\varepsilon\in\R_+$, we first compute
$$\chi_0 = {\arg\max}_{\chi\in\R^k, ||\chi||_0<\varepsilon} L_{\CE}(\F_{\Theta}(x+\chi),  l_x)$$
with PGD~\cite{M2017}
and then use
\begin{equation}
\label{eq-AT}
L_{\AT}(x,\Theta) = L_{\CE}(\F_{\Theta}(x+\chi_0),  l_x)
\end{equation}
to measure the robustness of $\F_{\Theta}$ at $x$.

We need a training set $T$ to find the adversarial parameters. The training procedure consists of two phases.
In the first pre-training phase, the loss function
\begin{equation}
\label{eq-AT11}
-\sum_{x\in T} L_{\AT}(x, \Theta)
\end{equation}
%
is used to reduce the adversarial accuracy of $\F_{\Theta}$.
In the second  main training phase, the loss function
\begin{equation}
\label{eq-AT12}
\frac{\sum_{x\in T}L_{\CE}(\F_{\Theta}(x),  l_x)}{\sum_{x\in T}L_{\AT}(x,\Theta)}
\end{equation}
is used to promote the accuracy and keep the low-level of adversarial accuracy
of $\F_{\Theta}$, which corresponds to  formula \eqref{eq-apa}.

We will compute a more general $L_{\infty}$ norm parameter perturbation.
Let $\Delta \in \R_{+}^k$ and $\Delta_i$ the $i$-th coordinate of $\Delta$.
Then the $L_{\infty}$ parameter perturbation will be found in
$$B_{\infty}(\Theta,\Delta)=\{\Theta_a\in\R^k\,|\,|\Theta_a-\Theta|_i\le\Delta_i,\ \forall\ i\in[k]\}.$$
It is clear that the usual $L_{\infty}$ norm parameter perturbation is a special
case of the above general case.
We use this general form, because we want to include more types of
parameter perturbations which are given in section \ref{sec-exp20}.
A sketch of the algorithm is given below.

{
\begin{algorithm}[H]
\caption{Attack under $L_{\infty}$ norm}
\label{alg-ap1}
\begin{algorithmic}
\REQUIRE ~~\\
The  parameter set $\Theta$ of $\F$;\\
The hyper-parameters: $\alpha\in\R_+$,   $\Delta\in\R_+^k$, $n_1, n_2 \in \N$;\\
A training set $T$.\\
\ENSURE An adversarial parameter set $\Theta_a$ in $B_{\infty}(\Theta,\Delta)$.\\
Let $i=0$, $\Theta_a=\Theta$.\\
For all $i\in[n_1+n_2]$:\\
\quad If $i<n_1$:\\
\quad\quad $L=-\sum_{x\in T}\frac{1}{|T|}L_{\AT}(x,\Theta_a)$.\\
\quad Else:\\
\quad\quad $L=\frac{\sum_{x\in T}L_{\CE}(\F_{\Theta_a}(x),  l_x)}{\sum_{x\in T}L_{\AT}(x,\Theta_a)}$.\\
\quad $\widetilde{\Theta}=\Theta_a+\alpha\bigtriangledown L$.\\
\quad $\Theta_a=\Proj(\widetilde{\Theta},B_{\infty}(\Theta,\Delta))$.\\
Output: $\Theta_a$.\\
\end{algorithmic}
\end{algorithm}
}

\begin{remark} We give more details for the algorithm.

\noindent
(1).   $\Proj(\widetilde{\Theta},B_{\infty}(\Theta,\Delta))$ maps $\widetilde{\Theta}$ into
$B_{\infty}(\Theta,\Delta)$ as follows:  for $i\in[k]$:

{\parskip=2pt
\quad If $\widetilde{\Theta}_i>\Theta_i+\Delta_i$,  $\Proj(\widetilde{\Theta},B_{\infty}(\Theta,\Delta))_i=\Theta_i+\Delta_i$;

\quad If $\widetilde{\Theta}_i<\Theta_i-\Delta_i$,
$\Proj(\widetilde{\Theta},B_{\infty}(\Theta,\Delta))_i=\Theta_i-\Delta_i$;

\quad If $\widetilde{\Theta}_i+\Delta_i>\Theta_i>\widetilde{\Theta}_i-\Delta_i$,  $\Proj(\widetilde{\Theta},B_{\infty}(\Theta,\Delta))_i=\Theta_i$.
}

\noindent
(2). We will reduce the training steps $\alpha$ with the training going.
%
%
\end{remark}

\subsection{Algorithms for other kinds of adversarial parameters}
\label{sec-alg2}

The algorithm to find adversarial parameters under other norms
and robustness measures can be developed similarly.
In what below, we show how to compute adversarial parameters under $L_{0}$ norm,
which is different from other cases. The overall algorithm is similar to Algorithm \ref{alg-ap1}, except we use a new method to update the parameters.
Suppose that $\Theta = \{W_l,b_l\}_{l=1}^{L}$ is the parameter to be updated and
the value of  $L$ in Algorithm \ref{alg-ap1} is found.
We will show how to update the parameters. We only change
some weight matrices $W_l$ as follows.
\begin{itemize}
\item
Randomly select two entries  $w_1$ and $w_2$ of $W_{l}$ until $(\frac{\nabla L}{\nabla w_1}-\frac{\nabla L}{\nabla w_2})(w_1-w_2)>0$ is satisfied.

\item
Exchange $w_1$ and $w_2$ in $W_l$ to obtain the new parameters.
\end{itemize}
It is clear that the change will make $L$ become smaller.
In total, we update a given number of weight matrices,
and for each such matrix, we change a given percentage of its entries.
The details of the algorithm are omitted.
Note that the above parameter perturbation keeps the sparsity and the values of the entries
of the weight matrices.
As a consequence the Proj operator in the algorithm can be taken as the identity map.

The adversarial parameters defined in section \ref{sec-def2} can also be obtained similarly.
For instance, to compute the adversarial parameters for one sample $x$, we just need to let $T$ in Algorithm \ref{alg-ap1} to be $T=\{x\}$.

To compute  adversarial parameters for samples with a given label $l_0$,
by \eqref{eq-z20} we can use the following loss function
\begin{equation}
\label{eq-AT14}
\frac
{\sum_{x\in T} L_{\CE}(\F_{\Theta_a}(x),  l_x)
+ \sum_{x\in T\ \&\ l_x\ne l_0} L_{\AT}(x, \Theta_a)}
{\sum_{x\in T\ \&\ l_x=l_0}L_{\AT}(\F_{\Theta_a}(x),  l_x)}
\end{equation}
to increase the robustness and accuracy of samples whose labels are not $l_0$
and to reduce the accuracy for samples with labels $l_0$.

To compute  direct adversarial parameters for samples with label $l_0$,
by \eqref{eq-z21} we can use the following loss function
\begin{equation}
\label{eq-AT13}
\frac
{\sum_{x\in T\ \&\ l_x\ne l_0}
 (L_{\AT}(x, \Theta_a) + L_{\CE}(\F_{\Theta_a}(x),  l_x))}
{\sum_{x\in T\ \&\ l_x=l_0}L_{\CE}(\F_{\Theta_a}(x),  l_x)}
\end{equation}
to increase the robustness and accuracy of samples whose labels are not $l_0$
and to reduce the accuracy for samples with labels $l_0$.

\section{Existence of adversarial parameters}
\label{sec-t}
In this section, we will show that adversarial parameters
with high adversarial rates exist under certain conditions.

\subsection{Adversarial parameters to achieve low adversary accuracy}
\label{sec-t1}
In this section, we use the robustness radius in \eqref{eq-rm11} and the
adversarial accuracy in \eqref{eq-rm12} as the
robust measures, and hence existence of adversarial parameters implies
low adversary accuracies.

We introduce several notations.
Let $||x||_{-\infty}=\min_{i\in[n]}\{|x_i|\}$ for $x\in\R^n$,
and $||W||_{-\infty,2}=\min_{i\in[a]}\{||W^{(i)}||_2\}$ for $W\in\R^{a\times b}$,
where $W^{(i)}$ is the $i$-th row of $W$.
If $\F$ is a network, we use $\F_i(x)$ to denote the $i$-th coordinate of $\F(x)$.

In this section, we consider the following  network $\F_{\Theta}:\I^n\to\R^m$ with one hidden layer
\begin{equation}
\label{eq-dnnx}
\begin{array}{ll}
%
\F(x)= W_{2}\Relu(W_{1}x+b_1)+b_2,
\end{array}
\end{equation}
where
$W_{1}\in \R^{n_1\times n}, b_1\in\R^{n_1},
 W_{2}\in \R^{m\times n_1}, b_{2}\in \R^{m}$.
$\Theta=\{W_i,b_i\}_{i=1}^{2}\in\R^k$ is the parameter set of $\F_{\Theta}$, where $k=(n+m+1)n_1+m$.

The network defined in \eqref{eq-dnnx} has just one hidden-layer.
We will show that when the width of its hidden-layer is large enough, adversarial parameters exist under with certain conditions.

We will consider $L_{\infty}$ adversarial parameters.
For $\gamma\in\R_+$,  the hypothetical space for the adversarial networks of $\F_\Theta$ is
$$\H_{\gamma}(\Theta)=\{\F_{\Theta_a}\,|\,||\Theta_a-\Theta||_{\infty}< \gamma\}.$$

The following theorem shows the existence of adversarial parameters for a given sample $x_0$. The proof of the theorem is given in section \ref{sec-p1}.
\begin{theorem}
\label{th-1}
Let $\F_\Theta$  be a trained network with structure in \eqref{eq-dnnx},
which gives the correct label $l_{x_0}$ for a sample $x_0$.
Further assume the following conditions.
\begin{description}
\item[$C_1$.] Let $a,A\in\R_+$ such that $|\F_i(x)-\F_j(x)|<A$ for all $i,j\in[m]$ and $x\in S_{\infty}(x_0,a)=\{x\,|\, ||x-x_0||_\infty\le a\}$.

\item[$C_2$.] $||W_2^{(i)}-W_2^{(j)}||_{-\infty}>c$ for all $i,j\in[m]$, $i\ne j$.

\item[$C_3$.] At least $\eta  n_1$ coordinates of $|\Relu(W_1x+b_1)|$ are bigger than $b $, where $\eta \in(0,1)$ and $b \in\R_+$.
\end{description}
\noindent
%
For $\gamma,\epsilon\in\R_+$ such that $\epsilon<a$,
if  $n_1>\frac{2A}{\min\{\epsilon \gamma  (n-1), b \} c \eta }$,
then there exists an $\F_{\Theta_a}\in \H_{\gamma}(\Theta)$
such that $\widehat{\F}_{\Theta_a}(x_0) = l_{x_0}$ and
 $\F_{\Theta_a}$ has adversarial samples to $x_0$ in $S_{\infty}(x_0,\epsilon)$.
\end{theorem}

We have the following observations from Theorem \ref{th-1}.
\begin{cor}
\label{cor-11}
%
If the robustness radius in \eqref{eq-rm11} is used as the robustness measure,
then the adversarial rate of $\F_{\Theta_a}$ in Theorem \ref{th-1} is bigger than $1-\frac{\epsilon}{a}$.
\end{cor}

\begin{remark}
\label{rem-411}
From Theorem \ref{th-1}, if the width of $\F_\Theta$ is sufficiently large,
then  $\F_{\Theta}$ has adversarial parameters which is as close as possible to $\Theta$
and $\F_{\Theta_a}$ has adversarial samples which are as close as possible to $x_0$.
\end{remark}

\begin{remark}
\label{rem-412}
Since $\F_{\Theta}$ is a continuous function  in $\Theta$,
if $\Theta_a$ is an adversarial parameter set for $\Theta$ then
there exists a small sphere $S_a$ with $\Theta_a$ as center such that
all parameters in $S_a$ are also adversarial parameters for $\Theta$.
\end{remark}
From the above two remarks, we may say that adversarial parameters are inevitable
in this case.

The following theorem shows that when $n_1$ is large enough, adversarial parameters exist
for a distributions with high probability.
 The proof of the theorem is given in section \ref{sec-p1}.
%
\begin{theorem}
\label{th-2}
Let $\F_\Theta$ be a trained DNN with structure in \eqref{eq-dnnx}
and $S\subset\I^n$ the set of normal samples.
Further assume the following conditions.
\begin{description}
\item[$C_1$.]
 Let $A,a\in\R_+$  such that $|\F_i(x)-\F_j(x)|<A$ for all $i,j\in[m]$ and $x\in \bigcup_{x_0\in S} S_{\infty}(x_0,a)$.

\item[$C_2$.] $||W_2^{(i)}-W_2^{(j)}||_{-\infty}>c$ for all $i,j\in[m]$, $i\ne j$, where $c\in\R_+$.

\item[$C_3$.] For all $x\in S$, at least $\eta  n_1$ coordinates of $|\Relu(W_1x+b_1)|$ are bigger than $b $, where $\eta \in(0,1)$ $b \in\R_+$.

\item[$C_4$.] The dimension of $S$ is lower than $n-m$.
\end{description}
%
For $\epsilon,\gamma\in\R_+$ such that $\epsilon<a$,
 if  $n_1>\frac{2A}{\min\{\epsilon \gamma /m,b \} c\eta }$,
then there exists an $\F_{\Theta_a}\in \H_{\gamma}(\Theta)$ such that
the accuracy of ${\F}_{\Theta_a}$ over $D_x$ is greater than or equal to that of ${\F}_{\Theta}$ and
$$P_{x\sim D_x}({\F}_{\Theta_a}\ has\ an\ adversarial\ sample\ of\ x\ in\ S_{\infty}(x_,\epsilon))\ge 0.5.$$
\end{theorem}

\begin{cor}
Let the adversary accuracy of  $\F_\Theta$ be $\theta=R_3(\F_\Theta,D_x,\epsilon)$,
%
then Theorem \ref{th-2} implies that  there exists adversarial parameters $\Theta_{a}$ of $\Theta$ with adversarial rate at least $1-0.5/\theta$.
\end{cor}

\begin{remark}
The conditions of Theorems \ref{th-1} and \ref{th-2} can be satisfied for most DNNs.
The parameters $A$ and $a$ in Condition $C_1$ are clearly exist.
Since the training procedure usually terminates near a local minimum of the
loss function, the weights can be considered as random values~\cite{netl2},
and hence conditions $C_2$ and $C_3$ can be satisfied.
For practical examples such as MNIST and CIFAR-10, condition $C_4$
is clearly satisfied.
\end{remark}


\subsection{Adversarial parameters for DNNs}
\label{sec-t2}
In this section, we consider networks of the following form
\begin{equation}
\label{eq-dnnx2}
\renewcommand{\arraystretch}{1.2}
\begin{array}{ll}
x_0\in\I^n, n_{L+1}=m;\\
x_{l}=\Relu(W_{l}x_{l-1}) \in \R^{n},  W_{l}\in \R^{n\times n},   l\in[L]; \\
\F(x_0)=x_{L+1}=W_{L+1}x_{L}\in\R^m,
 W_{L+1}\in \R^{m\times n}.\\
\end{array}
\end{equation}
Let $\Theta=\{W_i\}_{i=1}^{L+1}$ be the parameters and $\Theta\in\R^k$, where $k=Ln^2+mn$.
%
We use $\F^l_{\Theta}(x)$ to represent the output of the $l$-th layer of $\F_{\Theta}(x)$ where $l\in[L]$.
We will show that, when $L$ becomes big,  adversarial parameters exist.

We first prove the existence of adversarial parameters for a given sample.
We use the following robustness measure for network $\F$ at
a sample $x$
%
$$R(\F,x)=\min_{l\ne l_x}\{\frac{|\F_{l_x}(x)-\F_{l}(x)|^2}
  {||\nabla(\F_{l_x}(x))-\nabla(\F_l(x))||_2^2}I(\F_{l_x}(x)>\F_{l}(x))\}.$$
%
It is easy to see that this is the square of $R_2(\F,x)$ in \eqref{eq-rm21} with $p=2$.

\begin{theorem}
\label{th-3}
Let $\F_\Theta$ be a trained network with structure in \eqref{eq-dnnx2},
which gives the correct label for a sample $x_0\in\R^n$.
Further assume the following conditions.
\begin{description}
\item[$C_1$.]
 $||\frac{\nabla \F_i(t)}{\nabla t}|_{t=x_0}||_2<\sqrt{A}$ for $i\in[m]$.
\item[$C_2$.]
$||\frac{\nabla \F^l(t)}{\nabla t}|_{t=x_0}||_{-\infty,2}>b $ for $l\in[L]$.
\item[$C_3$.]
 $||\frac{\nabla \F_i(t)-\nabla\F_j(t)}{\nabla \F^l(t)}|_{t=x_0}||_{-\infty}>c$ for $i,j\in[m]$, $i\ne j$ and $l\in[L]$.
\item[$C_4$.]
 For $l\in[L]$, $\frac{\nabla \F^l(t)}{\nabla t}|_{t=x_0}$ has a column $L_l$ such that the angle between $L_l$ and $\F^l(x_0)$ is bigger than $\alpha$ and smaller than $\pi-\alpha$, where $\alpha\in[0,\pi/2]$.
\end{description}
\noindent
%
Then for $\gamma \in\R_+$,
there exists an $\F_{\Theta_a}\in \H_{\gamma}(\Theta)$
such that $\F_{\Theta_a}(x_0) = l_{x_0}$ and
$$R(\F_{\Theta_a},x_0)\le(1-\eta)R(\F_\Theta,x_0)$$
where $\eta=\frac{\gamma ^2((L-1)(\sin(r) c b )^2+c^2+(2\sin(r) b )^2)}{4A+ \gamma ^2((L-1)(\sin(r) c b )^2+c^2+(2\sin(r) b )^2)}$.
In other words, there exists an $\F_{\Theta_a}$ with adversarial rate $\ge \eta$.
\end{theorem}

The proof of the theorem is given in section \ref{sec-p2}.
As a consequence of Theorem \ref{th-3},
there exist adversarial parameters for sample $x_0$, whose robustness measure is as small as possible.
\begin{cor}
\label{cor-31}
%
For $\rho\in(0,1)$,
if $ L \ge \frac{4(1-\rho) A}{\rho(\gamma \sin(r) c b )^2}+1$,
then the adversarial rate  of $\F_{\Theta_a}$ is $\ge 1-\rho$.
\end{cor}
\begin{cor}
\label{cor-32}
For $\tau\in(0,1)$ satisfying $\tau< R(\F,x_0)$,
if $L>\frac{4A(R(\F,x_0)/\tau-1)}{(\gamma \sin(r) c b )^2}+1$, then $R(\F_{\Theta_a},x_0)\le \tau$.
\end{cor}


To find adversarial parameters for samples under a distribution $D_x$, we use the following robustness measure  for $\F$:
$$R(\F,D_x)=\frac{\int_{x\sim D_x}
\min_{j\ne l_x}\{||\F_{l_x}(x)-\F_{j}(x)||_2^2\,I(\F_{l_x}(x)>\F_j(x))\}\d x}{\int_{x\sim D_x}\max_{j\ne l_x}\{||\nabla\F_{l_x}(x)-\nabla\F_{j}(x)||_2^2\}\d x}.$$
This is a variant of $R_4(\F,D_x)$ in \eqref{eq-rm22} with $p=2$.
%
%
The following theorem shows that adversarial parameters exist for this robustness measure.
The proof of the theorem is given in section \ref{sec-p2}.
\begin{theorem}
\label{th-4}
{Let $\F_\Theta$ be a trained DNN with structure in \eqref{eq-dnnx}
and $S\subset\I^n$ the set of normal samples satisfying distribution $D_x$.
Further assume the following conditions.
\begin{description}
\item[$C_1$.]
 $||\frac{\nabla \F_i(t)}{\nabla t}|_{t=x}||_2<\sqrt{A}$ for all samples $x\in S $ and $i\in[m]$.
\item[$C_2$.]
  $P_{x\sim D_x}(\forall l\ne l_x,\ ||\frac{\nabla (\F_l(t)-\F_{l_x}(t))}{\nabla \F^l(t)}|_{t=x}||_{-\infty}>c_l)>\alpha_l$, where  $l\in[L]$ and $c_l,\alpha_l\in\R_+$.
\item[$C_3$.]
 $P_{x\sim D_x}(||v\frac{\nabla \F^l(t)}{\nabla t}|_{t=x}||_{\infty}\ge d_l||v||_{\infty})>\beta_l$ for $\forall v\in\R^{1\times n}$, where $l\in[L]$ and $d_l,\beta_l\in\R_+$.
\item[$C_4$.]
 $\{\F^l(x)\}_{x\in S }$ is in a low dimensional subspace of $\R^n$
and   $||\F^l(x)||_0>\gamma_l/n$, where $l\in[L]$, $x\in S $, and $\gamma_l\in\R_+$.
\end{description}
\noindent
For $\gamma\in\R_+$, let $\H(\gamma)$ be the set of networks in $\H_{\gamma}(\Theta)$,
whose  accuracies are equal to or larger than that of $\F_\Theta$.
%
Then
{\small
$$
\min_{\widetilde{\F}\in H(\gamma )}\{R(\widetilde{\F},D_x)\}\le
\displaystyle(1-\rho)R(\F,D_x)$$
}
where $\rho=\frac{(\gamma  c_1)^2 \alpha_1 \gamma_1+\sum_{i=2}^{L}(\gamma  c_i d_{i-1})^2 \gamma_i (\alpha_i+\beta_{i-1}-1)+\beta_L(d_L\gamma )^2}{4A+(\gamma  c_1)^2 \alpha_1 \gamma_1+\sum_{i=2}^{L}(\gamma  c_i d_{i-1})^2 \gamma_i (\alpha_i+\beta_{i-1}-1)+\beta_L(d_L\gamma )^2}\displaystyle$.
In other words, there exists an $\F_{\Theta_a}$ with adversarial rate $\ge \rho$.}
\end{theorem}

We can make the robustness of the perturbed network as small as possible.
\begin{cor}
\label{cor-41}
In Theorem \ref{th-4},
if $\alpha,\beta,c,d\in\R_+$ satisfy
$\alpha_l>\alpha$, $\beta_l>\beta$, $c_l>c$, $d_l>d$, $\gamma_l>\gamma_{low}$ for $l\in[L]$,
then
$$\min_{\widetilde{\F}\in H(\gamma )}\{R(\widetilde{\F},D_x)\}\le
(1-\frac{(\gamma  c)^2 \alpha \gamma_{low}+(L-1)(\gamma  c d)^2 \gamma_{low} (\alpha+\beta-1)+\beta(d\gamma )^2}{4A+(\gamma  c)^2 \alpha \gamma_{low}+(L-1)(\gamma  c d)^2 \gamma_{low} (\alpha+\beta-1)+\beta(d\gamma )^2})R(\F,D_x).$$

Furthermore, for $\rho\in(0,1)$, if
$L > 1 + \frac{4(1-\rho) A }
{\rho((\gamma  c d)^2 \gamma_{low} (\alpha+\beta-1))}$, then
there exists an $\F_{\Theta_a}\in\H(\gamma)$ whose adversarial rate is $\ge 1-\rho$.

Furthermore, for $\tau\in(0,1)$ satisfying  $\tau<R(\F,D_x)$,
if  $L>\frac{4A(R(\F,D_x)/\tau-1)}{(\gamma  c d)^2 \gamma_{low} (\alpha+\beta-1)}+1$,
then there exists an $\F_{\Theta_a}\in\H(\gamma)$
such that $R(\F_{\Theta_a},D_x)\}\le \tau$.
\end{cor}

\begin{remark}
\label{rem-441}
From Corollary \ref{cor-41}, if the depth of the DNN is sufficiently large,
then there exist adversary parameters such that the attacked network
has robustness measure as small as possible.
\end{remark}

\begin{remark}
\label{rem-442}
In practical computation, we use a finite set $T$ of samples satisfying $D_x$
and ${R}(\F,D_x)$ is approximately
computed as $\widetilde{R}(\F,T)=1/|T|\sum_{x\in T} R(\F,x)$.
Since $\F_{\Theta}$ is a continued function in $\Theta$,
if $\Theta_a$ is an adversarial parameter for $\Theta$
and  $\widetilde{R}(\F,T)$ is used as the robustness measure,
then there exists a small sphere $S_a$ with $\Theta_a$ as center such that
all parameters in $S_a$ are also adversarial parameters for $\Theta$.
\end{remark}
The above remarks show that adversary parameters are inevitable in certain sense.

\begin{remark}
In the model \eqref{eq-dnnx2}, two simplifications are made.
However, the results proved in this section can be generalized to general DNNs.
First, the bias vectors are not considered, which can be included as parts of the
weight matrices by extending the input space slightly, similar to~\cite{N2014}.
Second, it is assumed that $n_l = n$ for $l\in[L]$.
This assumption could be removed by assuming  $n = \max_{l\in[l]} n_i$.
\end{remark}

\begin{remark}
Using $R_4(\F,D_x)$, results in theorem \ref{th-4} cannot be obtained yet.
But, we will use numerical experiments to show that the result is also valid
for $R_4(\F,D_x)$.
\end{remark}

\section{Experimental results}

\subsection{The setting}
\label{sec-oma11}

We use two networks: VGG19~\cite{VGG} and Resnet56~\cite{Res}.
We write VGG19 as $\F_V$ and Resnet56 as $\F_R$,
which are trained with the adversarial training~\cite{M2017}.
The experimental results are for the CIFAR-10 dataset.

We use both the adversary accuracy in \eqref{eq-rm12}
and the approximate robust radius   in \eqref{eq-rm22} to compute the adversarial rate.
For a given data set $T$, the adversarial accuracy   defined in \eqref{eq-rm12}
can be approximately computed with  PGD~\cite{M2017} as follows
\begin{equation*}
\widetilde{R}_3(\F_\Theta,T,\epsilon)=1/|T| \sum_{x\in T} I(\widehat{\F}_{\Theta}(x')=l_x)
\end{equation*}
where $x'=\arg\max_{|\widetilde{x}-x|_\infty \le \epsilon} L_{\CE}(\F_\Theta(\widetilde{x}),l_x)$.
In the experiment, we set $\epsilon=8/255$.
The approximate robust radius in \eqref{eq-rm22} can be computed as follows
$$\widetilde{R}_4(\F,T)=\frac{1}{T}\sum_{x\in T}
\min_{l\ne l_x}\{\frac{\F_{l_x}(x)-\F_l(x)}{||\nabla\F_{l_x}(x)-\nabla\F_l(x)||_1}I(\F_{l_x}(x_0)>\F_{l}(x_0))\}$$
where the $L_1$-norm is used,
since we consider $L_\infty$ adversarial samples.

The accuracies, adversarial accuracies, and AARs of $\F_V$ and $\F_R$ under the $L_{\infty}$ norm attack are given in Table \ref{tab-mr1},
which are about the state of the art results for these DNNs.
\begin{table}[H]
\centering
\begin{tabular}{|c|c|c|c|}
    \hline
    Net & AC  & AA & AAR\\
    $\F_V$ & 80$\%$ & 45$\%$ & 0.0770\\
    $\F_R$ & 83$\%$ & 49$\%$ & 0.0194\\
    \hline
\end{tabular}
\caption{Results  for $\F_V$ and $\F_R$ on CIFAR-10. AC: accuracy, AA: adversarial accuracy, AAR: approximate accurate radius}
\label{tab-mr1}
\end{table}

\subsection{Adversarial parameter attack}
\label{sec-exp20}

Let $\Theta$ be the parameter set of $\F_V$ or $\F_R$, and  two kinds of parameter perturbation attacks will be carried out:

{\bf $L_{\infty,\gamma}$ perturbation} for $\gamma\in\R_+$:
We consider parameter perturbations in $B_{\infty}(\Theta,\Delta_\gamma)$,
where $\Delta_\gamma =(\gamma|\theta_1|,\ldots,\gamma|\theta_k|)$
for  $\Theta = (\theta_1,\ldots,\theta_k)$.
In other words, $\gamma$ is the perturbation ratio.
Also, the BN-layers will  be changed to compute this kind of perturbations.

{\bf $L_{0,k}$ perturbation} for $k\in\N_{>0}$: $k$ weight matrices are perturbed and
$\max\{400, 1\% \#W_l\}$ pairs of weights are changed for $\F_V$ with the method given in section \ref{sec-alg2} ($1\% \#W_l$ pairs of weights are changed for $\F_R$),
where $\#W_l$ is the number of entries of $W_l$.
The BN-layers will not be changed to compute this kind of perturbations.

We set  $\gamma_{\rm{low}}$ in Remark \ref{rem-12} to be $90\%$,
that is, if the accuracy of the perturbed network has
is less than $90\%$ of that of the original DNN, then the attack is considered failed.

\subsubsection{Random parameter perturbation}

We do random parameter perturbations and will use them
as bases for comparisons. The results are given in Table \ref{tab-mr15}.

\begin{table}[H]
\centering
\begin{tabular}{|l|c|c|l|}
  \hline
  Attack & AC  & AA &AR  \\
  No attack & 80$\%$ & 45$\%$&0 \\
  $L_{\infty,0.02}$ & 80$\%$ & 39$\%$ &0.13\\
  $L_{\infty,0.04}$ & 80$\%$ & 37$\%$ &0.17\\
  $L_{\infty,0.06}$ & 79$\%$ & 36$\%$ &0.2\\
  $L_{\infty,0.08}$ & 78$\%$ & 35$\%$ &0.22\\
  $L_{\infty,0.10}$ & 78$\%$ & 34$\%$ &0.24\\
  $L_{0,8}$ & 67$\%$ & 23$\%$ &0.41(fail)\\
  $L_{0,12}$ & 61$\%$ & 20$\%$ &0.42(fail)\\
  $L_{0,16}$& 57$\%$ & 19$\%$  &0.41(fail)\\
  \hline
\end{tabular}
\caption{Random perturbations for $\F_V$ . AC: accuracy, AA: adversarial accuracy, AR: adversarial rate}
\label{tab-mr15}
\end{table}

For the $L_{\infty}$ perturbations, the accuracies are kept high,
but the robustness does not decrease much, so the adversarial rates are low.
For the $L_{0}$  perturbations, the accuracy decreases too much  and are considered failed attacks.
In either case, random perturbations are not good adversarial parameters.
Thus adversarial parameters are sparse around the trained parameters, which is similar with adversarial samples~\cite{neth1}.

\begin{table}[H]
\centering
\begin{tabular}{|l|c|c|c|}
  \hline
  Attack & AC  & AA &AR  \\
  No attack & 83.1$\%$ & 48.5$\%$&0 \\
  $L_{\infty,0.02}$ & 82.9$\%$ & 48.0$\%$ &0.004\\
  $L_{\infty,0.04}$ & 82.7$\%$ & 48.0$\%$ &0.011\\
  $L_{\infty,0.06}$ & 82.3$\%$ & 47.5$\%$ &0.022\\
  $L_{\infty,0.08}$ & 81.7$\%$ & 46.6$\%$ &0.039\\
  $L_{\infty,0.10}$ & 81.0$\%$ & 45.5$\%$ &0.061\\
  $L_{0,10}$ & 81.5$\%$ & 47.8$\%$ &0.016\\
  $L_{0,20}$ & 80.7$\%$ & 45.8$\%$ &0.054\\
  $L_{0,30}$ & 80.5$\%$ & 45.7$\%$ &0.057\\
  \hline
\end{tabular}
\caption{Random perturbations for $\F_R$. AC: accuracy, AA: adversarial accuracy, AR: adversarial rate}
\label{tab-mr151}
\end{table}

Results of random perturbations for $\F_R$ are given in Table \ref{tab-mr151}.
From the results, we can see that network $\F_R$ is much more robust
against random parameter perturbations than $\F_V$.

\subsubsection{Adversarial parameter attack on $\F_V$ and $\F_{R}$}

We use algorithms in section \ref{sec-alg} to create
adversarial parameters.
The training set $T$ contains 500 samples for which $\F$ give the correct label.
The average results are given in Table \ref{tab-mr16}.

\begin{table}[H]
\centering
\begin{tabular}{|l|c|c|c|c|c|}
  \hline
  \multirow{2}{*}{Attack}&
   \multirow{2}{*}{AC}
       & \multicolumn{2}{c|}{AA in \eqref{eq-rm12}}
       & \multicolumn{2}{c|}{ARR in \eqref{eq-rm22}}  \\ \cline{3-6}
       &   & $\widetilde{R}_3(\F,T)$ &AR  &$\widetilde{R}_4(\F,T)$& AR \\ \hline
  No attack & 80$\%$ & 45$\%$&0&0.0770&0 \\
$L_{\infty,0.02}$ & 78$\%$ & 38$\%$ &0.15&0.0667&0.13\\
$L_{\infty,0.04}$ & 77$\%$ & 30$\%$ &0.32&0.0481&0.36\\
$L_{\infty,0.06}$ & 76$\%$ & 22$\%$ &0.49&0.0372&0.49\\
$L_{\infty,0.08}$ & 76$\%$ & 10$\%$ &0.74&0.0195&0.71\\
$L_{\infty,0.10}$ & 77$\%$ & 8$\%$&0.79&0.0143&0.78\\
$L_{0,8}$    & 72$\%$ & 27$\%$& 0.36 &0.0441&0.38\\
$L_{0,12}$    & 76$\%$ & 24$\%$&0.44 &0.0443&0.40\\
$L_{0,16}$  & 74$\%$ & 22$\%$&0.47 &0.0404&0.44\\
  \hline
\end{tabular}
\caption{Adversarial parameter attack for $\F_V$. AC: accuracy, AA: adversarial accuracy, AR: adversarial rate.}
\label{tab-mr16}
\end{table}

%
Comparing Tables \ref{tab-mr15} and \ref{tab-mr16},
we can see that algorithms in section \ref{sec-alg} can be used to  create
good adversarial parameters, especially for the  $L_{\infty}$ attack.
From Figure  \ref{fig-per3}, we can see that
 the adversarial rate  and adversarial accuracy for the $L_{\infty,\gamma}$ attack
are near linear in $\gamma$ when $\gamma$ is small, and is gradually stabilized with the increase of $\gamma$.
Also when $\gamma$ is very small, say $\gamma=0.02$,
the adversarial parameter attacks do not create good results,
which means the network is approximately safe against these attacks for $\gamma\le0.02$.
For $L_{0}$ attacks, we can see that the accuracies are increased lots
comparing to the random perturbation and
adversarial parameters are obtained successfully.
Also, for the two kinds of robustness-measurements, the adversarial rates   are very close.
For network $\F_R$, similar results are obtained and are given in Table \ref{tab-mr16res}.

\begin{figure}[H]
\centering
\hspace{2mm}
\includegraphics[scale=0.4]{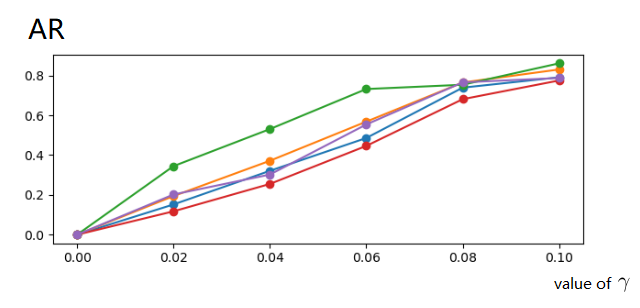}
\label{fig-per1}
%
\hspace{2mm}
\includegraphics[scale=0.4]{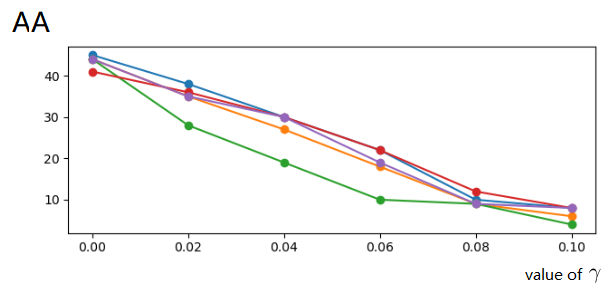}
\caption{Left: Relation between $\gamma$ and  the adversarial rate.
Right: Relation between $\gamma$ and adversarial accuracy.}
\label{fig-per3}
\end{figure}


\begin{table}[H]
\centering
\begin{tabular}{|l|c|c|c|c|c|}
  \hline
  \multirow{2}{*}{Attack}&
   \multirow{2}{*}{AC}
       & \multicolumn{2}{c|}{AA in \eqref{eq-rm12}}
       & \multicolumn{2}{c|}{ARR in \eqref{eq-rm22}}  \\ \cline{3-6}
       &   & $\widetilde{R}_3(\F,T)$ &AR  &$\widetilde{R}_4(\F,T)$& AR \\ \hline
  No attack & 83$\%$ & 49$\%$&0&0.0194&0 \\
$L_{\infty,0.02}$ & 84$\%$ & 39$\%$ &0.20&0.0151&0.22\\
$L_{\infty,0.04}$ & 85$\%$ & 27$\%$ &0.45&0.0127&0.35\\
$L_{\infty,0.06}$ & 86$\%$ & 14$\%$ &0.72&0.0092&0.53\\
$L_{\infty,0.08}$ & 87$\%$ & 6$\%$ &0.87&0.0064&0.67\\
$L_{\infty,0.10}$ & 87$\%$ & 1$\%$ &0.98&0.0044&0.77\\
{ $L_{0,10}$ }   & 80$\%$ & 26$\%$ & 0.45 &0.0193&0\\
{ $L_{0,20}$ }   & 78$\%$ & 18$\%$ & 0.59&0.0187&0.03\\
{ $L_{0,30}$}    & 72$\%$ & 9$\%$&0.71 &0.0125&0.33\\
  \hline
\end{tabular}
\caption{Adversarial parameter attack for $\F_R$. AC: accuracy, AA: adversarial accuracy, AR: adversarial rate.}
\label{tab-mr16res}
\end{table}

\subsubsection{Affect of network depth and width on the adversarial parameter attack}

We check how the network depth and width affect on the adversarial parameter attack.
We use the $L_{\infty,\gamma}$ adversarial parameter attack for $\gamma=0.02,0.04$.
Let $\F^k_V$ be the network which has the same width with $\F_V$ but has $k$ more layers, $\F_V(\alpha)$  the network which has the same depth with $\F_V$ but has $\alpha$ times  width as $\F_V$.
 The results are given in Table \ref{tab-bj1}.
We can see that when the depth becomes larger, the attack becomes easier.
This validates the results in section \ref{sec-t2},
for instance Corollary \ref{cor-41}, where it shows that when the depth
of the network becomes large, adversarial parameters exist.

The attack is much less  sensitive to the width.
The reason may be that there exist much redundancy on the width, similar to the results in \cite{ry1,ry3,ry4}, and the redundance can lead to limited search directions in the feature space and poor generalization performance, as shown in \cite{ry2}, so the attack is hard to improve when the width becomes larger.

\begin{table}[H]\small
\centering
\begin{tabular}{|c|c|c|c|c|c|c|c|c|c|}
  \hline
  \multirow{2}{*}{Network}
       & \multicolumn{2}{c|}{$\gamma=0$}
       & \multicolumn{3}{c|}{$\gamma=0.02$}
       & \multicolumn{3}{c|}{$\gamma=0.04$}  \\ \cline{2-9}
        &AC & AA  &AC  & AA  &AR  &AC  &  AA &AR \\\hline
  $\F_V$      &80$\%$  &45$\%$    &78$\%$    &38$\%$       &   0.15   & 77$\%$   &30$\%$       &0.32\\
  $\F^8_V$    &80$\%$  &47$\%$    &76$\%$    &37$\%$       &   0.20   & 78$\%$   &32$\%$       & 0.31\\
  $\F^{16}_V$ &78$\%$  &44$\%$    &74$\%$    &35$\%$       &   0.19   & 75$\%$   &27$\%$       & 0.37\\
  $\F^{24}_V$ &79$\%$  &43$\%$    &73$\%$    &30$\%$       &   0.28   & 73$\%$   &20$\%$       & 0.49\\
  $\F^{32}_V$ &76$\%$  &44$\%$    &72$\%$    &28$\%$       &   0.34   & 71$\%$   &19$\%$       & 0.53\\
  $\F_V(1.25)$&80$\%$  &42$\%$    &77$\%$    &37$\%$       &   0.12   & 76$\%$   &29$\%$       & 0.29   \\
  $\F_V(1.5)$ &81$\%$  &41$\%$    &78$\%$    &36$\%$       &    0.12  & 77$\%$   &30$\%$       & 0.26   \\
  $\F_V(2)$   &78$\%$  &44$\%$    &76$\%$    &39$\%$       &  0.11    & 74$\%$   &31$\%$       & 0.28   \\
  $\F_V(2.5)$ &80$\%$  &44$\%$    &79$\%$    &35$\%$       &    0.20  & 76$\%$   &30$\%$       & 0.30   \\
   \hline
\end{tabular}
\caption{Affect of width and depth on adversarial parameter attack for $\F_V$.
AC: accuracy, AA: adversarial accuracy, AR: adversarial rate.}
\label{tab-bj1}
\end{table}

We can  use $\widetilde{R}_4(\F,T)$ to measure the robustness and similar results are obtained, which are given in Table \ref{tab-bj}.

\begin{table}[H]\small
\centering
\begin{tabular}{|c|c|c|c|c|c|c|}
  \hline
  \multirow{2}{*}{Network}
       & \multicolumn{1}{c|}{$\gamma=0$}
       & \multicolumn{2}{c|}{$\gamma=0.02$}
       & \multicolumn{2}{c|}{$\gamma=0.04$}  \\ \cline{2-6}
       &$\widetilde{R}_4(\F,T)$ & $\widetilde{R}_4(\F,T)$ & AR  & $\widetilde{R}_4(\F,T)$   &AR  \\ \hline
  $\F_V$             & 0.0770   & 0.0667 &0.13 & 0.0481 &0.36 \\
  $\F^8_V$           & 0.0776   & 0.0686 & 0.11& 0.0534 & 0.30\\
  $\F^{16}_V$        & 0.0815   & 0.0652 &0.19 & 0.0479 & 0.40\\
  $\F^{24}_V$        & 0.0817   & 0.0630 & 0.21& 0.0388 & 0.49\\
  $\F^{32}_V$        & 0.0808   & 0.0608 & 0.23& 0.0392 & 0.48\\
  $\F_V(1.25)$       & 0.0750   & 0.0663 & 0.11& 0.0524 & 0.29\\
  $\F_V(1.5)$        & 0.0763   & 0.0670 &0.12 & 0.0563 & 0.25 \\
  $\F_V(2)$          & 0.0750   & 0.0650 & 0.13& 0.0499 & 0.32\\
  $\F_V(2.5)$        & 0.0775   & 0.0678 &0.12 & 0.0489 & 0.35\\
   \hline
\end{tabular}
\caption{Affect of width and depth on adversarial parameter attack for $\F_V$.
AR: adversarial rate}
\label{tab-bj}
\end{table}

\subsection{Direct adversarial parameters}

We give experimental results for direct adversarial parameters
introduced in section \ref{sec-def2}.
We try to decrease the accuracies for samples with label 0 and keep the
accuracies and robustness for other samples.
The experimental results are for the network $\F_V$ and CIFAR-10
and are given in Table \ref{tab-mr1ga}.

\begin{table}[H]
\centering
\begin{tabular}{|c|c|c|c|c|}
  \hline
  Attack & AC$_1$  & AA$_1$&AC$_0$ &AR  \\
  $L_{\infty,0.02}$ & 77$\%$ & 35$\%$ &11$\%$&0.65\\
  $L_{\infty,0.04}$ & 78$\%$ & 40$\%$ & 3$\%$&0.83\\
  $L_{\infty,0.06}$ & 79$\%$ & 42$\%$ & 1$\%$&0.92\\
  $L_{\infty,0.08}$ & 80$\%$ & 43$\%$ & 1$\%$&0.95\\
  $L_{\infty,0.1}$& 80$\%$ & 45$\%$ & 1$\%$&0.99\\
  \hline
\end{tabular}
\caption{Direct adversarial parameter attack for $\F_V$. AC$_1$ and AA$_1$ are for samples with label $\ne0$, AC$_0$ is the accuracy for samples with label $0$.}
\label{tab-mr1ga}
\end{table}

Comparing to Tables \ref{tab-mr1ga} and \ref{tab-mr16},
we can see that direct adversarial parameters for a given label
are much easier to compute than adversarial parameters.
High quality direct adversarial parameters
can be obtained by using perturbation ratios $6\%-10\%$.
Results for network $\F_R$ are given in Table \ref{tab-mr2ga},
from which we can see that it is slightly more difficult to attack $\F_R$.
\begin{table}[H]
\centering
\begin{tabular}{|c|c|c|c|c|}
  \hline
  Attack & AC$_1$  & AA$_1$&AC$_0$ &AR  \\
  $L_{\infty,0.02}$ & 82$\%$ & 46$\%$ &52$\%$&0.38\\
  $L_{\infty,0.04}$ & 81$\%$ & 47$\%$ & 35$\%$&0.57\\

  $L_{\infty,0.06}$ & 81$\%$ & 47$\%$ & 10$\%$&0.83\\
$L_{\infty,0.08}$ & 81$\%$ & 46$\%$ & 4$\%$&0.90\\
    $L_{\infty,0.1}$& 80$\%$ & 46$\%$ & 1$\%$&0.93\\
  \hline
\end{tabular}
\caption{Direc adversarial parameter attack for $\F_R$.  AC$_1$ and AA$_1$ are for samples with label $\ne0$, AC$_0$ is the accuracy for samples with label $0$.}
\label{tab-mr2ga}
\end{table}

\subsection{Adversarial parameters for a given sample}
We give experimental results for adversarial parameters
for a given sample introduced in section \ref{sec-def2}.
$\widetilde{R}_2(\F,x)=\min_{i\ne l_x}\{\frac{\F_{l_x}(x)-\F_i(x)}{||\nabla\F_{l_x}(x)-\nabla\F_i(x)||_1}\}$ is used to measure the robustness of $\F$ at sample $x$.
Let $S$ be a subset of the test set containing 100 samples
such that $\F$ gives the correct label for all of them
and all samples in $S$ are robust in that,
no adversarial samples were found using PGD-10 with  $L_{\infty}=\frac{8}{255}$.

\begin{table}[H]
\centering
\begin{tabular}{|c|c|c|c|c|}
  \hline
        Attack   & $\widetilde{R}_2(\F,x)$  & $N_1$& $N_2$ & AR  \\
before attack    & 0.078  &100&100&0\\
$L_{\infty,0.02}$& 0.016 & 0 &100&0.79\\
$L_{\infty,0.04}$& 0.010 & 0 &100&0.87\\
$L_{\infty,0.06}$& 0.008 & 0 &100&0.89\\
$L_{\infty,0.08}$& 0.006 & 0 &100&0.92\\
$L_{\infty,0.1}$ & 0.005 & 0 &100&0.94\\
  \hline
\end{tabular}
\caption{Adversarial parameter attack to $\F_V$ for a given sample. AR: adversarial rate}
\label{tab-wh}
\end{table}

For each sample $x\in S$, we compute $L_{\infty,\gamma}$  adversarial parameters
and the average results are given in Table \ref{tab-wh},
where $N_1$ is the number of robust samples,
and $N_2$ the number of samples which are given the correct labels.
%
Comparing to Tables \ref{tab-mr1ga}, \ref{tab-mr16}, and \ref{tab-wh},
we can see that adversarial rates for a single sample are
about the same as that for a given label.
Similar results for network $\F_R$ are given in Table \ref{tab-whres}.

\begin{table}[H]
\centering
\begin{tabular}{|c|c|c|c|c|}
  \hline
        Attack   & $\widetilde{R}_2(\F,x)$  & $N_1$& $N_2$ & AR  \\
before attack    & 0.057  &100&100&0\\
$L_{\infty,0.02}$& 0.008 & 0 &100&0.86\\
$L_{\infty,0.04}$& 0.008 & 0 &100&0.86\\
$L_{\infty,0.06}$& 0.008 & 0 &100&0.86\\
$L_{\infty,0.08}$& 0.008 & 0 &100&0.86\\
$L_{\infty,0.1}$ & 0.008 & 0 &100&0.86\\
  \hline
\end{tabular}
\caption{Adversarial parameter attack to $\F_R$ for a given sample.
AR: Adversarial Rate}
\label{tab-whres}
\end{table}

\section{Proofs for the theorems in section \ref{sec-t}}
\label{sec-p}

\subsection{Proofs of Theorems \ref{th-1} and \ref{th-2}}
\label{sec-p1}
We introduce several notations.
Let $||x||_{-\infty}=\min_{i\in[n]}\{|x_i|\}$ for $x\in\R^n$,
and $||W||_{-\infty,2}=\min_{i\in[a]}\{||W^{(i)}||_2\}$ for $W\in\R^{a\times b}$,
where $W^{(i)}$ is the $i$-th row of $W$.
If $\F$ is a network, we use $\F_i(x)$ to denote the $i$-th coordinate of $\F(x)$.
A lemma is proved first.

\begin{lemma}
\label{buc-wd}
Let $v\in\R^n$ and $v\ne 0$. Then there exists a vector $w\in\R^n$ such that
$w\bot v$, $||w||_{\infty}=1$ and $||w||_2\ge\sqrt{n-1}$.
\end{lemma}
\begin{proof}
Let $S=\arg\min_{S\subseteqq [n]}\{|\sum_{i\in S} |v_i|-\sum_{i\in [n]\setminus S} |v_i||\}$.
We can assume   $\sum_{i\in S} |v_i|-\sum_{i\in [n]\setminus S} |v_i|=k\ge0$.
For any $j\in S$ such that $v_j\ne0$ and $S_1=S/\{j\}$,
we have $|\sum_{i\in S_1} |v_i|-\sum_{i\in [n]/S_1} |v_i|| = |2|v_j|-k|\ge k$,
which means $k\le |v_j|$.

We now define $w\in\R^n$.
Set $w_i=1$ if $v_i=0$.
Select a $j\in S$ such that $v_j\ne0$ and let $w_i=\sign(v_i)$ if $i\in S/\{j\}$ and $v_i\ne0$, and $w_j=\frac{-\sum_{i\in S} |v_i|+\sum_{i\in [n]\setminus S} |v_i|+|v_j|}{v_j}$.
For $i\in[n]\setminus S$, let $w_i=-\sign(v_i)$ if $v_i\ne0$.
It is easy to check that  $||w||_{\infty}=1$, $||w||_2\ge\sqrt{n-1}$, and $w\bot v$.
The lemma is proved.
\end{proof}

%
%
%
%
We now prove Theorem \ref{th-1}.
\begin{proof}
By Lemma \ref{buc-wd}, there exists a vector $v\in\R^n$ such that $v\bot x_0$, $||v||_2\ge\sqrt{n-1}$ and $||v||_{\infty}=1$.
Moreover, we can assume that at least $\eta n_1/2$ coordinates of $\Relu(W_1(x+\epsilon v)+b_1)$ are bigger than $b $.
If this is not valid, we just need to change $v$ to $-v$,
and then
$\Relu(W_1(x+\epsilon v)+b_1)+\Relu(W_1(x-\epsilon v)+b_1)\ge2\Relu(W_1x+b_1)$,
since $\Relu(x)+\Relu(y)\ge \Relu(x+y)$ for all $x,y\in\R$.
By condition $C_3$, at least $\eta n_1$ coordinates of $2\Relu(W_1x+b_1)$ are bigger than $2b $, but fewer than $\eta n_1/2$ coordinates of $\Relu(W_1(x+\epsilon v)+b_1)$ are bigger than $b $, so  at least $\eta n_1/2$ coordinates of $\Relu(W_1(x-\epsilon v)+b_1)$ are bigger than $b $.

Let $l_2\in[m]$ such that $l_2\ne l_{x_0}$,
$\overline{W}_2=-(W_2^{(l_{x_0})}-W_2^{(l_2)})\in\R^{1\times n_1}$,
$U_v\in\R^{n_1\times n}$ all of whose rows are $\gamma v^\tau$ (the transposition of $v$), and $U=\diag(\sign(\overline{W}_2))\in\R^{n_1\times n_1}$.
Let $\widetilde{W}_1=W_1+UU_v$ and
$$\widetilde{\F}(x)=W_2\Relu(\widetilde{W}_1x+b_1)+b_2.$$
We will show that $\widetilde{\F}$ satisfies the condition of the theorem.

Since $v\bot x_0$, we have $\widetilde{\F}(x_0)=\F(x_0)$ and $\widetilde{\F}$ gives the correct label for $x_0$.
Since $||v||_{\infty}=1$, we have $||\widetilde{W}_1-W_1||_{\infty}=||UU_v||_{\infty}=||U_v||_{\infty}=||\gamma v||_{\infty}\le \gamma$, and thus $\widetilde{\F}(x)\in H_{\gamma}(\Theta)$.

So it suffices to show that  $\widetilde{\F}(x_0+\epsilon v)$ will not give $x_0+\epsilon v$ label $l_{x_0}$, which means that $\widetilde{\F}$ has adversarial samples to $x_0$ in $S_{\infty}(x_0,\epsilon)$.
Since
\begin{equation*}
\begin{array}{ll}
&\widetilde{\F}(x_0+\epsilon v)\\
=&W_2\Relu(\widetilde{W}_1(x_0+\epsilon v)+b_1)+b_2\\
=&W_2\Relu(W_1(x_0+\epsilon v)+b_1+\epsilon UU_vv)+b_2\\
\end{array}
\end{equation*}
we have
\begin{equation*}
\begin{array}{ll}
&\widetilde{\F}_{l_{x_0}}(x_0+\epsilon v)-\widetilde{\F}_{l_2}(x_0+\epsilon v)\\
=&\F_{l_{x_0}}(x_0+\epsilon v)-\F_{l_2}(x_0+\epsilon v)+\\
&\overline{W}_2(\Relu(W_1(x+\epsilon v)+b_1)-\Relu(W_1(x+\epsilon v)+b_1+\epsilon UU_vv)).\\
\end{array}
\end{equation*}
Since $e(\Relu(f)-\Relu(e+f))=-|f|(|\Relu(e)-\Relu(e+f))|$ for all $e,f\in\R$, we have
\begin{equation*}
\begin{array}{ll}
&-\overline{W}_2(\Relu(W_1(x+\epsilon v)+b_1)-\Relu(W_1(x+\epsilon v)+b_1+\epsilon UU_vv))\\
=&|\overline{W}_2|(|\Relu(W_1(x+\epsilon v)+b_1)-\Relu(W_1(x+\epsilon v)+b_1+\epsilon UU_vv)|).\\
\end{array}
\end{equation*}
Since  $\epsilon UU_vv=\epsilon\gamma ||v||_2^2\sign(\overline{W}_2)$ and $||v||_2\ge n-1$,  each weight of $|\epsilon UU_vv|$ is at least $\epsilon\gamma (n-1)$.
Note that if $e>0$ and $f\in\R$, then $|\Relu(e)-\Relu(e-f)|\ge \min\{e,|f|\}$.
As a consequence,
 if $i$ satisfies $(\Relu(W_1(x+\epsilon v)+b_1))_i>b $, then $(|\Relu(W_1(x+\epsilon v)+b_1)-\Relu(W_1(x+\epsilon v)+b_1+\epsilon UU_vv)|)_i\ge \min\{\epsilon \gamma (n-1),b \}$. %
Since at least $\eta n_1/2$ coordinates of $\Relu(W_1(x+\epsilon v)+b_1)$ are bigger than $b $,
  we have $||\Relu(W_1(x+\epsilon v)+b_1)-\Relu(W_1(x+\epsilon v)+b_1+\epsilon UU_vv)||_1\ge \eta n_1/2\min\{\epsilon \gamma (n-1),b \}$.
By condition $C_2$, it is easy see
\begin{equation*}
\begin{array}{ll}
&-\overline{W}_2(\Relu(W_1(x+\epsilon v)+b_1)-\Relu(W_1(x+\epsilon v)+b_1+\epsilon UU_vv))\\
=&|\overline{W}_2|(|\Relu(W_1(x+\epsilon v)+b_1)-\Relu(W_1(x+\epsilon v)+b_1+\epsilon UU_vv)|)\\
\ge& ||\overline{W}_2||_{-\infty}||\Relu(W_1(x+\epsilon v)+b_1)-\Relu(W_1(x+\epsilon v)+b_1+\epsilon UU_vv)||_1\\
\ge&\min\{\epsilon \gamma (n-1),b \} c n_1 \eta/2,
\end{array}
\end{equation*}
that is, $\overline{W}_2(\Relu(W_1(x+\epsilon v)+b_1)-\Relu(W_1(x+\epsilon v)+b_1+\epsilon UU_vv)) \le -\min\{\epsilon \gamma (n-1),b \} c n_1 \eta/2$.
By condition $C_1$, we have $\F_{l_{x_0}}(x_0+\epsilon v)-\F_{l_2}(x_0+\epsilon v)\le A$.
Then we have
\begin{equation*}
\begin{array}{ll}
&\widetilde{\F}_{l_{x_0}}(x_0+\epsilon v)-\widetilde{\F}_{l_2}(x_0+\epsilon v)\\
=&\F_{l_{x_0}}(x_0+\epsilon v)-\F_{l_2}(x_0+\epsilon v)+\\
&\overline{W}_2(\Relu(W_1(x+\epsilon v)+b_1)-\Relu(W_1(x+\epsilon v)+b_1+\epsilon UU_vv))\\
\le& A-\min\{\epsilon \gamma (n-1),b \} c n_1 \eta/2
<0.
\end{array}
\end{equation*}
Thus
if $n_1>\frac{2A}{\min\{\epsilon \gamma  (n-1), b \} c \eta }$,
then
$\widetilde{\F}_{l_{x_0}}(x_0+\epsilon v)-\widetilde{\F}_{l_2}(x_0+\epsilon v)<0$
and the label of $\widetilde{\F}(x_0+\epsilon v)$ is not $l_{x_0}$.
The theorem is proved.
\end{proof}

%
%
%
%
%
%
We now prove Theorem \ref{th-2}.
\begin{proof}
By condition $C_4$, for $l\in[m]$, there exist $v_l\in\R^n$ such that $v_l\bot S$, $v_l\bot v_k$ for $l\ne k$, $||v_l||_2=1$.
{Then  $||v_l||_{\infty}\le 1$.}

By condition $C_3$,  at least $\eta n_1/2$ coordinates of $\Relu(W_1(x+\epsilon v_{l_x})+b_1)$ are bigger than $b$
or at least $\eta n_1/2$ coordinates of $\Relu(W_1(x-\epsilon v_{l_x})+b_1)$ are bigger than $b$, similar to the proof of Theorem \ref{th-1}.

For convenience, we write $G(x,y):(\R^n,\R)\to \R$ as $G(x,y)=||\sign(x-yI_n)||_0$,
where $I_n$ is the vector with entries 1. It is easy to see that, $G(x,b)$ is the number of coordinates of $x$ that are bigger than $b$.
So we have $G(\Relu(W_1(x+\epsilon v_{l_x})+b_1),b )\ge\eta n_1/2$
 or $G(\Relu(W_1(x-\epsilon v_{l_x})+b_1),b )\ge\eta n_1/2$ for all $x$, and hence for $l\in[m]$,  we have
{\small
$$P_{x\sim D_x}(G(\Relu(W_1(x+\epsilon v_{l_x})+b_1),b )\ge\eta n_1/2
\hbox{ or } G(\Relu(W_1(x-\epsilon v_{l_x})+b_1),b )\ge\eta n_1/2 \,|\,l_x=l)=1.$$
}
For events $e$ and $f$, $P(e\ \hbox{or}\ f)\le P(e)+P(f)$. We thus have
$P_{x\sim D_x}(G(\Relu(W_1(x+\epsilon v_{l_x})+b_1),b )\ge\eta n_1/2\,|\,l_x=l)\ge0.5$
or
$P_{x\sim D_x}(G(\Relu(W_1(x-\epsilon v_{l_x})+b_1),b )\ge\eta n_1/2\,|\,l_x=l)\ge0.5$.
Without loss of generality,
we can assume that for any $l\in[m]$,
$P_{x\sim D_x}(G(\Relu(W_1(x+\epsilon v_{l_x})+b_1),b )\ge\eta n_1/2\,|\,l_x=l)\ge0.5$.
Therefore,
\begin{equation*}
\begin{array}{ll}
&P_{x\sim D_x}(G(\Relu(W_1(x+\epsilon v_{l_x})+b_1),b )\ge\eta n_1/2)\\
=&\sum_{l\in[m]}P_{x\sim D_x}(l_x=l)P_{x\sim D_x}(G(\Relu(W_1(x+\epsilon v_{l_x})+b_1),b )\ge\eta n_1/2\,|\,l_x=l)\\
\ge&0.5\sum_{l\in[m]}P_{x\sim D_x}(l_x=l) \\
=&0.5.
\end{array}
\end{equation*}

For $l\in[m]$, let $\overline{W}_2^{(l)}=W_2^{(l)}-W_2^{(l+1)}$, where $W_2^{(m+1)}=W_2^{(1)}$.
Now assume $U^l_v\in\R^{n_1\times n}$, whose rows are all $\gamma v_l$, and $U_l=\diag(\sign(\overline{W}_2^{(l)}))$.
Let $\widetilde{W}_1=W_1+\frac{1}{m}\sum_{l=1}^{m}U_lU^l_v$, and
$$\widetilde{\F}(x)=W_2\Relu(\widetilde{W}_1x+b_1)+b_2.$$
We will show that $\widetilde{\F}(x)$ satisfies the conditions of the theorem.

It is easy to see that $\widetilde{\F}$ is in $\H_{\gamma}(\Theta)$.
For any $x\in S$,  $\widetilde{W}_1x=W_1x+\frac{1}{m}\sum_{l=1}^{m}(U_lU^l_v)x=W_1x$, which means $\widetilde{\F}(x)=\F(x)$ and
the accuracy of $\widetilde{\F}$ over $D_x$ is equal to that of $\F$.

Let $x\in S$ satisfy  $G(\Relu(W_1(x+\epsilon v_{l_x})+b_1),b )>\eta n_1/2$
and $l_2\ne l_x$. By conditions $C_1$ and $C_2$ and similar to the proof of Theorem \ref{th-1}, we have
\begin{equation*}
\begin{array}{ll}
&\widetilde{\F}_{l_x}(x+\epsilon v_{l_x})-\widetilde{\F}_{l_2}(x+\epsilon v_{l_x})\\
=&\F_{l_x}(x+\epsilon v_{l_x})-\F_{l_2}(x+\epsilon v_{l_x})+\\
&W^{(l_x)}_c(\Relu(W_1(x+\epsilon v_{l_x})+b_1)-\Relu(W_1(x+\epsilon v_{l_x})+b_1+\epsilon U_lU^l_vv/m))\\
\le& A-\min\{\epsilon \gamma /m,b \} c n_1 \eta/2
<0.
\end{array}
\end{equation*}
Thus $\widetilde{\F}$ does not give label $l_x$ to $x+\epsilon v_{l_x}$
and $\widetilde{\F}$ has an adversarial sample of $x$ in $S_{\infty}(x,\epsilon)$.
Furthermore, since $P_{x\sim D_x}(G(\Relu(W_1(x+\epsilon v_{l_x})+b_1),b )\ge\gamma n_1/2)>0.5$,
  we have
  $$P_{x\sim D_x}(\widetilde{\F}\ has\ an\ adversarial\ sample\ of\ x\ in\ S_{\infty}(x_,\epsilon))>0.5.$$
  The theorem is proved.
\end{proof}

 \subsection{Proofs of Theorems \ref{th-3} and \ref{th-4}}
\label{sec-p2}

We first prove two lemmas.
\begin{lemma}
\label{m2}
For $l\in[m]$, let $S_l$ be a non-empty bounded closed subset of $\R^{n\times n}$
such that $W\in S_l$ implies $-W\in S_l$.
Also let $S_0$ be a non-empty bounded closed subset of $\R^{1\times n}$ such that $x\in S_0$ implies $-x\in S_0$.
Let $U_0\in\R^{1\times n}$ and $U_l\in\R^{n\times n}$ for $l\in[m]$.
Define maps:
$T_0(x):\R^{1\times n}\to\R^{1\times n}$ by $T_0(x)=x\prod_{l=1}^{m}U_l$,
and for $l\in[m]$, $T_l(W):\R^{n\times n}\to\R^{1\times n}$ by $T_l(W)=U_0(\prod_{j=1}^{l-1}U_j)W(\prod_{j=l+1}^{m}U_j)$.
 Then
$$\max_{x_l\in S_l, \forall 0\le l\le m}\{||\prod_{l=0}^{m}(x_l+U_l)||_2^2\}\ge||\prod_{l=0}^{m} U_l||^2_2+\sum_{l=0}^{m}\max_{x_l\in S_l}\{||T_l(x_l)||_2^2\}$$
\end{lemma}
\begin{proof}
For  $l\in[m]$, let $u_l=\arg\max_{x\in S_l}||T_l(x)||_2^2$, which exists because $S_l$ is bounded and closed. Then
\begin{equation*}
\begin{array}{ll}
&\max_{x_l\in S_l, \forall 0\le l\le m}\{||\prod_{l=0}^{m}(x_l+U_l)||_2^2\}\\
\ge&\max_{x_l\in \{u_l,-u_l\}, \forall 0\le l\le m}\{||\prod_{l=0}^{m}(x_l+U_l)||_2^2\}\\
=&\frac{1}{2^{m+1}}\sum_{x_l\in\{u_l,-u_l\}, \forall 0\le l\le m}||\prod_{l=0}^{m}(x_l+U_l)||_2^2\\
=&\sum_{M_l\in\{u_l,U_l\}, \forall 0\le l\le m} ||\prod_{l=0}^{m}M_l||_2^2\\
\ge&||\prod_{l=0}^{m} U_l||_2+\sum_{l=0}^{m}||T_l(u_l)||_2^2\\
=&||\prod_{l=0}^{m} U_l||_2+\sum_{l=0}^{m}\max_{x_l\in S_l}\{||T_l(x_l)||_2^2\}
\end{array}
\end{equation*}
The lemma is proved.
\end{proof}

\begin{lemma}
\label{m2-b}
For $l\in[m]$, let $S_l$ be a non-empty closed subset of bounded functions from $\I ^k$ to $R^{n\times n}$ such that  $c(x)\in S_l$ implies  $-c(x)\in S_l$.
Also let $S_0$ be a non-empty closed subset of bounded functions from $\I ^k$ to $\R^{1\times n}$ such that $c(x)\in S_0$ implies $-c(x)\in S_0$.
Assume $U_0(x):\I ^k\to\R^{1\times n}$, $U_l(x):\I ^k\to\R^{n\times n}$ for $l\in[m]$.
Define maps:
$T_0(c,x):(S_0,\I ^k)\to\R^{1\times n}$ by $T_0(c,x)=c(x)\prod_{l=1}^{m}U_l(x)$;
and  $T_l(c,x):(S_l,\I ^k)\to\R^{1\times n}$  by $T_l(x)=(\prod_{j=0}^{l-1}U_j(x))c(x)(\prod_{j=l+1}^{m}U_j(x))$ for $l\in[m]$.
Then we have
{\small
$$\max_{c_l(x)\in S_l, \forall 0\le l\le m}
\displaystyle\{\int_{x\sim D_x}||\prod_{l=0}^{m}(c_l(x)+U_l(x))||_2^2\displaystyle\}
\ge
\int_{x\sim D_x}||\prod_{l=0}^{m} U_l(x)||^2_2+\sum_{l=0}^{m}\max_{c_l(x)\in S_l}\{\int_{x\sim {D_x}}||T_l(c_l,x)||_2^2\}.$$
}
\end{lemma}
\begin{proof}
Denote $\overline{c_l}(x)=\arg\max_{c\in S_l}||\int_{x\sim D_x}T_l(c,x)||_2^2$,
which must exist because $S_l$ is closed and its elements are bounded. Then
\begin{equation*}
\begin{array}{ll}
&\max_{c_l(x)\in S_l, \forall 0\le l\le m}\{\int_{x\sim D_x}||\prod_{l=0}^{m}(c_l(x)+U_l(x))||_2^2\}\\
\ge&\max_{c_l(x)\in \{\overline{c_l}(x),-\overline{c_l}(x)\}, \forall 0\le l\le m}\{\int_{x\sim D_x}||\prod_{l=0}^{m}(c_l(x)+U_l(x))||_2^2\}\\
\ge&\frac{1}{2^{m+1}}\sum_{c_l\in\{\overline{c_l}(x),-\overline{c_l}(x)\}, \forall 0\le l\le m}\int_{x\sim D_x}||\prod_{l=0}^{m}(c_l(x)+U_l(x))||_2^2\\
=&\sum_{M_l(x)\in\{\overline{c_l}(x),U_l(x)\}, \forall 0\le l\le m} ||\int_{x\sim D_x}\prod_{l=0}^{m}M_l(x)||_2^2\\
\ge&\int_{x\sim D_x}||\prod_{l=0}^{m} U_l(x)||^2_2+\sum_{l=0}^{m}\int_{x\sim D_x}||T_l(\overline{c_l},x)||_2^2\\
=&\int_{x\sim D_x}||\prod_{l=0}^{m} U_l||_2+\sum_{l=0}^{m}\max_{c_l\in S_l}\{\int_{x\sim D_x}||T_l(c_l,x)||_2^2\}.
\end{array}
\end{equation*}
The lemma is proved.
\end{proof}

%
%
%
%
%
%
We now prove Theorem \ref{th-3}.
\begin{proof}
Let $S_l$  be the set of $Q\in\R^{n\times n}$ satisfying $||Q||_{\infty}\le\gamma $ and $Q\F^{l-1}(x_0)=0$ for $l\in[L+1]$.
Note that $Q$ satisfies $n(l+1)$ linear equations and has $n^2$ variables.
Since  $n\gg l$ in DNNs, we can assume that $S_l$ is not empty.
Let $\FB$ be the set of networks $\widetilde{\F}:\R^n\to\R^m$ satisfying
$$\widetilde{\F}(x)=\widetilde{W}_{L+1}\sigma(\widetilde{W}_{L}
  \sigma(\widetilde{W}_{L-1}\sigma(\dots \sigma(\widetilde{W}_{1}x))))$$
where $\widetilde{W}_l-W_l\in S_l$ for all $l\in[L+1]$.
We have
$\widetilde{\F}^l(x_0)=\F^l(x_0)$ for $l\in[L+1]$. So $\widetilde{\F}(x_0) = l_{x_0}$ if $\widetilde{\F}\in \FB$.
It is also easy to see that $\FB\subset \H_{\gamma}(\Theta)$.
So it suffices to prove
$$\min_{\widetilde{\F}\in \FB}\{\frac{R(\widetilde{\F},x_0)}{R(\F,x_0)}\}\le1-\frac{\gamma ^2((L-1)(\sin(r) c b )^2+c^2+(2\sin(r) b )^2)}{4A+ \gamma ^2((L-1)(\sin(r) c b )^2+c^2+(2\sin(r) b )^2)}.$$

Let $l_2=\arg\min_{i\ne l_{x_0}}\{\frac{|\F_{l_{x_0}}(x_0)-\F_{i}(x_0)|^2}{||\nabla(\F_{l_{x_0}}(x_0))-\nabla(\F_i(x_0))||_2^2}I(\F_{l_{x_0}}(x_0)>\F_{i}(x_0))\}$. Then
$$R(\F,x_0)=\frac{|\F_{l_{x_0}}(x_0)-\F_{l_2}(x_0)|^2}{||\nabla(\F_{l_{x_0}}(x_0))-\nabla(\F_{l_2}(x_0))||_2^2}.$$ Then for all $\widetilde{\F}\in \FB$, we have
$$R(\widetilde{\F},x_0)\le\frac{|\widetilde{\F}_{l_x}(x_0)-\widetilde{\F}_{l_2}(x_0)|^2}{||\nabla(\widetilde{\F}_{l_x}(x_0))-\nabla(\widetilde{\F}_{l_2}(x_0))||_2^2}
=\frac{|\F_{l_x}(x_0)-\F_{l_2}(x_0)|^2}{||\nabla(\widetilde{\F}_{l_x}(x_0))-\nabla(\widetilde{\F}_{l_2}(x_0))||_2^2}$$
So we have
\begin{equation}
\label{eq-pr31}
\renewcommand{\arraystretch}{1.5}
\begin{array}{ll}
&\min_{\widetilde{\F}\in \FB}\{\frac{R(\widetilde{\F},x_0)}{R(\F,x_0)}\}\\
\le&\min_{\widetilde{\F}\in\FB}\{\frac{|\F_{l_x}(x_0)-\F_{l_2}(x_0)|^2}{||\nabla(\widetilde{\F}_{l_x}(x_0))-\nabla(\widetilde{\F}_{l_2}(x_0))||_2^2}/\frac{|\F_{l_x}(x)-\F_{l_2}(x)|^2}{||\nabla(\F_{l_x}(x))-\nabla(\F_{l_2}(x))||_2^2}\}\\
\le&\min_{\widetilde{\F}\in\FB}\{\frac{||\nabla(\F_{l_x}(x_0))-\nabla(\F_{l_2}(x_0))||_2^2}{||\nabla(\widetilde{\F}_{l_x}(x_0))-\nabla(\widetilde{\F}_{l_2}(x_0))||_2^2}\}.
\end{array}
\end{equation}

To prove the theorem, we will first find a lower bound for $\max_{\widetilde{\F}\in\FB}\{||\nabla(\widetilde{\F}_{l_x}(x_0))-\nabla(\widetilde{\F}_{l_2}(x_0))||_2\}$.
Let $J_{l}(x)=\diag(\sign(\F^l(x))):\ \R^n\to\R^{n\times n}$ for $l\in[L]$.
Then  
$$\frac{\nabla\F_i(x)}{\nabla x}=W^{i}_{L+1}(J_{L}(x)W_L)(J_{L-1}(x)W_{L-1})\dots(J_{1}(x)W_1).$$
Let $A_l(x)=\frac{\nabla \F^L(t)}{\nabla \F^l(t)}|_{t=x}$ and $B_l(x)=\frac{\nabla \F^l(t)}{\nabla t}|_{t=x}$ for $l\in[L]$. Then we have
\begin{equation*}
\label{eq-pr3111}
\begin{array}{ll}
A_l(x)=(J_{L}(x)W_L)(J_{L-1}(x)W_{L-1})\dots(J_{l+1}(x)W_{l+1})\\ B_l(x)=(J_{l}(x)W_l)(J_{l-1}(x)W_{l-1})\dots(J_{1}(x)W_1).
\end{array}
\end{equation*}
Let $A_l=A_l(x_0)$, $B_l=B_l(x_0)$, $J_l=J_l(x_0)$.
Then for all $\widetilde{\F}\in \FB$,
$$\frac{\nabla \widetilde{\F}_i(x)}{\nabla x}|_{x=x_0}=\widetilde{W}_{L+1}^i(J_{L}\widetilde{W}_L)(J_{L-1}\widetilde{W}_{L-1})
\dots(J_{1}\widetilde{W}_1).$$
Denote $\overline{W}_l=\widetilde{W}_l-W_l\in S_l$ for $l\in[L+1]$. We have
\begin{equation*}
\begin{array}{ll}
&\nabla(\widetilde{\F}_{l_x}(x_0))-\nabla(\widetilde{\F}_{l_2}(x_0))\\
=&(\widetilde{W}_{L+1}^{(l_x)}-\widetilde{W}_{L+1}^{(l_2)})(J_{L}\widetilde{W}_L)
(J_{L-1}\widetilde{W}_{L-1})\dots(J_{1}\widetilde{W}_1)\\
=&(W_{L+1}^{(l_x)}-W_{L+1}^{(l_2)}+\overline{W}_{L+1}^{(l_x)}-\overline{W}_{L+1}^{(l_2)})(J_{L}(W_L+\overline{W}_L))(J_{L-1}(W_{L-1}+\overline{W}_{L-1}))\dots(J_{1}(W_1+\overline{W}_1)).\\
\end{array}
\end{equation*}
Now, let $\gamma(\overline{W}_1)=(W_{L+1}^{(l_x)}-W_{L+1}^{(l_2)})A_{1}J_{1}\overline{W}_1$, $M_i(\overline{W}_i)=(W_{L+1}^{(l_x)}-W_{L+1}^{(l_2)})A_{i}J_{i}\overline{W}_iB_{i-1}$, where $i\in\{2,3,\dots,L\}$, $M_{L+1}(\overline{W}_{L+1})=(\overline{W}_{L+1}^{(l_x)}-\overline{W}_{L+1}^{(l_2)})B_{L}$.
By Lemma \ref{m2}, we have
\begin{equation*}
\begin{array}{ll}
&\max_{\widetilde{\F}\in\FB}\{||\nabla(\widetilde{\F}_{l_x}(x_0))-\nabla(\widetilde{\F}_{l_2}(x_0))||^2_2\}\\
\ge& \max_{\overline{W}_j\in S_j, j\in[L+1]}\{||\nabla(\F_{l_x}(x_0))-\nabla(\F_{l_2}(x_0))||_2^2+ \sum_{i=1}^{L+1}||M_i(\overline{W}_i)||^2_2\}\\
=&||\nabla(\F_{l_x}(x_0))-\nabla(\F_{l_2}(x_0))||_2^2+ \sum_{i=1}^{L+1}\max_{\overline{W}_i\in S_i}\{||M_i(\overline{W}_i)||^2_2\}.
\end{array}
\end{equation*}

{It is easy to see for any $l\in[L]$, $(W_{L+1}^{(l_x)}-W_{L+1}^{(l_2)})A_{l}=\frac{\nabla \F_{l_x}(t)-\F_{l_2}(t)}{\nabla \F^l(t)}|_{t=x_0}$, so by condition $C_3$, we have $||(W_{L+1}^{(l_x)}-W_{L+1}^{(l_2)})A_{i}||_{-\infty}>c$.}
Let $S^1_l$ be the subset of $S_l$   containing those $C$ which have  at most one nonzero row.
Hence, for $x\in \R^{1\times n}$ and $M\in\R^{n\times n}$, if at most one row of $M$ is nonzero, we have $||xM||_{\infty}=\max_{i,j\in[n]}\{|x_iM_{i,j}|\}\ge||x||_{-\infty}||M||_{\infty}$, where $x_i$ is the $i$-th weight of $x$, $M_{i,j}$ is the weight of $M$ at $i$-th row and $j$-th column.
Thus
\begin{equation*}
\renewcommand{\arraystretch}{1.5}
\begin{array}{ll}
&\max_{\overline{W}_1\in S_1}\{||\gamma(\overline{W}_1)||_2\}\\
=&\max_{\overline{W}_1\in S_1} \{||(W_{L+1}^{(l_x)}-W_{L+1}^{(l_2)})A_{1}J_{1}\overline{W}_1||_2\}\\
\ge&\max_{\overline{W}_1\in S_1} \{||(W_{L+1}^{(l_x)}-W_{L+1}^{(l_2)})A_{1}J_{1}\overline{W}_1||_{\infty}\}\\
\ge&\max_{\overline{W}_1\in S^1_1} \{||(W_{L+1}^{(l_x)}-W_{L+1}^{(l_2)})A_{1}J_{1}\overline{W}_1||_{\infty}\}\\
\ge&||(W_{L+1}^{(l_x)}-W_{L+1}^{(l_2)})A_{1}||_{-\infty}\max_{\overline{W}_1\in S^1_1}\{||J_{1}\overline{W}_1||_{\infty}\}\\
\ge&\gamma  c.
\end{array}
\end{equation*}
Moreover, by condition $C_4$, there exists a column $L_{i-1}$ of $B_{i-1}$ such that $\pi-r\ge\alpha(\F^{i-1}(x_0),L_{i-1})\ge r$, where $\alpha(x,y)$ is the angle between $x,y$.
Therefore, there exists a vector $v_i\in\R^n$ such that $v\bot\F^{i-1}(x)$, $||v_i||_{\infty}=\gamma $ and consider condition  $C_2$, we have $\langle v_i, L_{i-1}\rangle=||v_i||_2||L_{i-1}||_2\cos(\pi/2-r)\ge\sin(r) b  \gamma $.

Then $\max_{\overline{W}_i\in S^1_i}||J_{i}\overline{W}_iB_{i-1}||_{\infty}\ge \sin(r) b  \gamma $,
because there must exist a $\overline{W}_i\in S^1_i$ whose   only nonzero row  is $v_i$ and $J_{i}\overline{W}_i=\overline{W}_i$.
For $l\in\{2,3,\dots,L\}$, {by condition $C_3$,}   we have
\begin{equation*}
\begin{array}{ll}
&\max_{\overline{W}_l\in S_l}\{||M_l(\overline{W}_l)||_2\\
=&\max_{\overline{W}_l\in S_l}\{||(W_{L+1}^{(l_x)}-W_{L+1}^{(l_2)})A_{i}J_{l}\overline{W}_lB_{l-1}||_2\}\\
\ge&\max_{\overline{W}_l\in S_l}\{||(W_{L+1}^{(l_x)}-W_{L+1}^{(l_2)})A_{l}J_{l}\overline{W}_lB_{l-1}||_{\infty}\}\\
\ge&\max_{\overline{W}_l\in S^1_l}\{||(W_{L+1}^{(l_x)}-W_{L+1}^{(l_2)})A_{l}J_{l}\overline{W}_lB_{l-1}||_{\infty}\}\\
\ge&||(W_{L+1}^{(l_x)}-W_{L+1}^{(l_2)})A_{1}||_{-\infty}\max_{\overline{W}_l\in S_l}||J_{l}\overline{W}_lB_{l-1}||_{\infty}\\
\ge& \gamma  \sin(r) c b .
\end{array}
\end{equation*}
Similarly,
\begin{equation*}
\begin{array}{ll}
&\max_{\overline{W}_{L+1}\in S_{L+1}}\{||M_{L+1}(\overline{W}_{L+1})||_2\}\\
=&\max_{\overline{W}_{L+1}\in S_{L+1}} \{||(\overline{W}_{L+1}^{(l_x)}-\overline{W}_{L+1}^{(l_2)})B_{L}||_2\}\\
\ge&\max_{\overline{W}_{L+1}\in S_{L+1}} \{||(\overline{W}_{L+1}^{(l_x)}-\overline{W}_{L+1}^{(l_2)})B_{L}||_{\infty}\}\\
\ge&2\sin(r) b  \gamma .\\
\end{array}
\end{equation*}
Then we obtain the desired lower bound:
\begin{equation*}
\begin{array}{ll}
&\max_{\widetilde{\F}\in\FB}\{||\nabla(\widetilde{\F}_{l_x}(x_0))-\nabla(\widetilde{\F}_{l_2}(x_0))||^2_2\}\\
\ge& \max_{\overline{W}_j\in S_j, j\in[L+1]}\{||\nabla(\F_{l_x}(x_0))-\nabla(\F_{l_2}(x_0))||_2^2+ \sum_{l=1}^{L+1}||M_l(\overline{W}_l)||^2_2\}\\
\ge&||\nabla(\F_{l_x}(x))-\nabla(\F_{l_2}(x))||_2^2+ \gamma ^2((L-1)(\sin(r) c b )^2+c^2+(2\sin(r) b )^2).
\end{array}
\end{equation*}

By condition $C_1$ and the lower bound just obtained, we have
\begin{equation*}
\begin{array}{ll}
&\min_{\widetilde{\F}\in\FB}\{\frac{||\nabla(\F_{l_x}(x_0))-\nabla(\F_{l_2}(x_0))||_2^2}{||\nabla(\widetilde{\F}_{l_x}(x_0))-\nabla(\widetilde{\F}_{l_2}(x_0))||_2^2}\}\\
\le&\frac{||\nabla(\F_{l_x}(x_0))-\nabla(\F_{l_2}(x_0))||_2^2}{||\nabla(\F_{l_x}(x_0))-\nabla(\F_{l_2}(x_0))||_2^2+ \gamma ^2((L-1)(\sin(r) c b )^2+c^2+(2\sin(r) b )^2)}\\
=&1-\frac{\gamma ^2((L-1)(\sin(r) c b )^2+c^2+(2\sin(r) b )^2)}{||\nabla(\F_{l_x}(x_0))-\nabla(\F_{l_2}(x_0))||_2^2+ \gamma ^2((L-1)(\sin(r) c b )^2+c^2+(2\sin(r) b )^2)}\\
\le&1-\frac{\gamma ^2((L-1)(\sin(r) c b )^2+c^2+(2\sin(r) b )^2)}{4A+ \gamma ^2((L-1)(\sin(r) c b )^2+c^2+(2\sin(r) b )^2)}.
\end{array}
\end{equation*}
The theorem is proved.
\end{proof}

We now prove Theorem \ref{th-4}.

\begin{proof}
The proof is similar to that of Theorem \ref{th-3}, so certain details are omitted.
Let $T_l=\{\F^{l-1}(x)\,|\,x\in S\}\subset\R$, and
$S_l$   the set of  $Q\in\R^{n\times n}$ such that $||Q||_{\infty}\le\gamma $ and $Qt=0$ for all $t\in T_{l}$ and $l\in[L+1]$. $S_l$ must contain non-zero elements because of condition $C_4$.

Let $\FB$ be the set of networks $\widetilde{\F}\in\R^n\to\R^m$
$$\widetilde{\F}(x)=\widetilde{W}_{L+1}\sigma(\widetilde{W}_{L}\sigma(\widetilde{W}_{L-1}
\sigma(\dots \sigma(\widetilde{W}_{1}x))))$$
where $\widetilde{W}_l-W_l\in S_l$.
Then for all $\widetilde{\F}\in \FB$, we have $\widetilde{\F}^l(x)=\F^l(x)$ for $l\in[L+1]$ and $x\in S$, so $\FB\subset H(\gamma )$. As a consequence,
\begin{equation*}
\begin{array}{ll}
&\int_{x\sim D_x}\min_{l\ne l_x}\{||\F_{l_x}(x)-\F_l(x)||_2^2I(\F_{l_x}(x)>\F_l(x))\}\d x\\
=&\int_{x\sim D_x}\min_{l\ne l_x}\{||\widetilde{\F}_{l_x}(x)-\widetilde{\F}_l(x)||_2^2I(\widetilde{\F}_{l_x}(x)>\widetilde{\F}_l(x))\}\d x.
\end{array}
\end{equation*}
Let $l_2=\arg\max_{l\ne l_x}\{||\nabla(\F_{l_x}(x))-\nabla(\F_l(x))||_2^2\}$. Then
\begin{equation*}
\begin{array}{ll}
&\int_{x\sim D_x}\max_{l\ne l_x}\{||\nabla(\F_{l_x}(x))-\nabla(\F_i(x))||_2^2\}\d x\\
=&\int_{x\sim D_x}||\nabla(\F_{l_x}(x))-\nabla(\F_{l_2}(x))||_2^2\d x\\
\end{array}
\end{equation*}
and
\begin{equation*}
\begin{array}{ll}
&\int_{x\sim D_x}\max_{l\ne l_x}\{||\nabla(\widetilde{\F}_{l_x}(x))-\nabla(\widetilde{\F}_i(x))||_2^2\}\d x\\
\ge&\int_{x\sim D_x}||\nabla(\widetilde{\F}_{l_x}(x))-\nabla(\widetilde{\F}_{l_2}(x))||_2^2\d x.\\
\end{array}
\end{equation*}
Therefore,
\begin{equation*}
\begin{array}{ll}
&\min_{\widetilde{\F}\in H(\gamma )}\{\frac{R(\widetilde{\F},D_x)}{R(\F,D_x)}\}\\
\le&\min_{\widetilde{\F}\in \FB}\{\frac{R(\widetilde{\F},D_x)}{R(\F,D_x)}\}\\
\le&\min_{\widetilde{\F}\in \FB}\{\frac{\int_{x\sim D_x}||\F_{l_x}(x)-\F_{l_2}(x)||_2^2\d x}{\int_{x\sim D_x}||\widetilde{\F}_{l_x}(x)-\widetilde{\F}_{l_2}(x)||_2^2\d x}\}.
\end{array}
\end{equation*}

We will estimate $\max_{\widetilde{\F}\in \FB}\{\int_{x\sim D_x}||\widetilde{\F}_{l_x}(x)-\widetilde{\F}_{l_2}(x)||_2^2\d x\}$.
Let $J_{l}(x)=\diag(\sign(\F^l(x)))\in\R^{n\times n}$, where $l\in[L]$.
Then  $\frac{\nabla\F_i(x)}{\nabla x}=W^{i}_{L+1}(J_{L}(x)W_L)(J_{L-1}(x)W_{L-1})\dots(J_{1}(x)W_1)$.
Also, for all $\widetilde{\F}\in \FB$,
$$\frac{\nabla \widetilde{\F}_i(x)}{\nabla x}=\widetilde{W}_{L+1}^i(J_{L}(x)\widetilde{W}_L)
(J_{L-1}(x)\widetilde{W}_{L-1})\dots(J_{1}(x)\widetilde{W}_1).$$
Denote $\overline{W}_i=\widetilde{W}_i-W_i\in S_i$. Then
{\small
\begin{equation*}
\renewcommand{\arraystretch}{1.5}
\begin{array}{ll}
&\nabla(\widetilde{\F}_{l_x}(x))-\nabla(\widetilde{\F}_{l_2}(x))\\
=
&(\widetilde{W}_{L+1}^{(l_x)}-\widetilde{W}_{L+1}^{(l_2)})
   (J_{L}(x)\widetilde{W}_L)(J_{L-1}(x)\widetilde{W}_{L-1})\dots(J_{1}(x)\widetilde{W}_1)\\
=&(W_{L+1}^{(l_x)}-W_{L+1}^{(l_2)}+\overline{W}_{L+1}^{(l_x)}-\overline{W}_{L+1}^{(l_2)})(J_{L}(x)(W_L+\overline{W}_L))(J_{L-1}(x)(W_{L-1}+\overline{W}_{L-1}))\dots(J_{1}(x)(W_1+\overline{W}_1))\\
\end{array}
\end{equation*}
}
Let $A_l(x)=\frac{\nabla \F^L(t)}{\nabla \F^l(t)}|_{t=x}$, $B_l(x)=\frac{\nabla \F^l(t)}{\nabla t}|_{t=x}$ where $l\in[L]$. Then

$A_l(x)=(J_{L}(x)W_L)(J_{L-1}(x)W_{L-1})\dots(J_{l+1}(x)W_{l+1})$ and

$B_l(x)=(J_{l}(x)W_i)(J_{l-1}(x)W_{l-1})\dots(J_{1}(x)W_1)$.

\noindent
Let $\gamma(x,\overline{W}_1)=(W_{L+1}^{(l_x)}-W_{L+1}^{(l_2)})A_{1}(x)J_{1}(x)\overline{W}_1$, $M_l(x,\overline{W}_l)=(W_{L+1}^{(l_x)}-W_{L+1}^{(l_2)})A_{l}(x)J_{l}(x)\overline{W}_lB_{l-1}(x)$ where $l\in\{2,3,\dots,L\}$, $M_{L+1}(x,\overline{W}_{L+1})=(\overline{W}_{L+1}^{(l_x)}-\overline{W}_{L+1}^{(l_2)})B_{L}(x)$.
By Lemma \ref{m2-b},
\begin{equation*}
\renewcommand{\arraystretch}{1.5}
\begin{array}{ll}
&\max_{\widetilde{\F}\in \FB}\{\int_{x\sim D_x}||\nabla(\widetilde{\F}_{l_x}(x))-\nabla(\widetilde{\F}_{l_2}(x))||^2_2\d x\}\\
\ge& \max_{\overline{W}_l\in S_l, l\in[L+1]}\{\int_{x\sim D_x}||\nabla(\F_{l_x}(x))-\nabla(\F_{l_2}(x))||_2^2+ \sum_{l=1}^{L+1}||M_l(\overline{W}_l)||^2_2\d x\}\\
=&\int_{x\sim D_x}||\nabla(\F_{l_x}(x))-\nabla(\F_{l_2}(x))||_2^2+ \sum_{l=1}^{L+1}\max_{\overline{W}_l\in S_l}\{\int_{x\sim D_x}||M_l(\overline{W}_l)||^2_2\d x\}.
\end{array}
\end{equation*}
Let $\overline{W}_{1}(k)\in\R^{n\times n}$ be the matrix whose $k$-th row is equal to $k$-th row of $\overline{W}_1$, and other rows are 0.
Let $(J_1(x))^{(k)}$ be the $k$-th row of $J_1(x)$.
Then
\begin{equation*}
\renewcommand{\arraystretch}{1.5}
\begin{array}{ll}
&\max_{\overline{W}_1\in S_1}\{\int_{x\sim D_x}||\gamma(x,\overline{W}_1)||^2_2\}\\
=&\max_{\overline{W}_1\in S_1} \{\int_{x\sim D_x}||(W_{L+1}^{(l_x)}-W_{L+1}^{(l_2)})A_{1}(x)J_{1}(x)\overline{W}_1||^2_2\d x\}\\
\ge&\max_{\overline{W}_1\in S_1} \{\int_{x\sim D_x}||(W_{L+1}^{(l_x)}-W_{L+1}^{(l_2)})A_{1}(x)J_{1}(x)\overline{W}_1||_{\infty}^2\d x\}\\
\ge&\max_{\overline{W}_1\in S_1, k\in [n]}\{\int_{x\sim D_x}(||(W_{L+1}^{(l_x)}-W_{L+1}^{(l_2)})A_{1}(x)||_{-\infty}||J_{1}(x)\overline{W}_1(k)||_{\infty})^2\d x\}\\
\ge&\max_{k\in [n]}\{\int_{x\sim D_x}(||(W_{L+1}^{(l_x)}-W_{L+1}^{(l_2)})A_{1}(x)||_{-\infty}I((J_{1}(x))^{(k)}\ne 0)\gamma )^2\d x\}.\\
\end{array}
\end{equation*}
By condition $C_2$, we know that $P_{x\sim D_x}(||(W_{L+1}^{(l_x)}-W_{L+1}^{(l_2)})A_{1}(x)||_{-\infty}>c_1)>\alpha_1$,
and by condition $C_4$ and the principle of drawer,  there exists a $k\in[n]$ such that
$P(x\sim D_x)(I((J_{1}(x))^{(k)}\ne 0)\ |\ ||(W_{L+1}^{(l_x)}-W_{L+1}^{(l_2)})A_{1}(x)||_{-\infty}>c)>\gamma$.
Thus  there exists a $k\in[n]$ such that
$$P(x\sim D_x)(I((J_{1}(x))^{(k)}\ne 0),||(W_{L+1}^{(l_x)}-W_{L+1}^{(l_2)})A_{1}(x)||_{-\infty}>c)>\gamma\alpha_1.$$
Then
\begin{equation*}
\renewcommand{\arraystretch}{1.5}
\begin{array}{ll}
&\max_{\overline{W}_1\in S_1}\{\int_{x\sim D_x}||\gamma(x,\overline{W}_1)||^2_2\}\\
\ge&\max_{k\in [n]}\{\int_{x\sim D_x}(||(W_{L+1}^{(l_x)}-W_{L+1}^{(l_2)})A_{1}(x)||_{-\infty}I((J_{1}(x))^{(k)}\ne 0)\gamma )^2\d x\}\\
\ge&\max_{k\in[n]}\{P(x\sim D_x)(I((J_{1}(x))^{(k)}\ne 0),||(W_{L+1}^{(l_x)}-W_{L+1}^{(l_2)})A_{1}(x)||_{-\infty}>c)(c\gamma )^2\}\\
\ge&(c_1 \gamma )^2 \gamma\alpha_1.\\
\end{array}
\end{equation*}

Let  $\widetilde{S}_l=\{x\in S_l\,|$ only\ one\ row\ of\ $x$\ is\ not\ $0\}$. Then   for $l\in\{2,3,\dots,L\}$, we have
\begin{equation*}
\renewcommand{\arraystretch}{1.5}
\begin{array}{ll}
&\max_{\overline{W}_l\in S_i}\{\int_{x \sim D_x}||M_l(x,\overline{W}_l)||^2_2\d x\}\\
=&\max_{\overline{W}_l\in S_l}\{\int_{x \sim D_x}||(W_{L+1}^{(l_x)}-W_{L+1}^{(l_2)})A_l(x)J_{l}(x)\overline{W}_lB_{l-1}(x)||^2_2 \d x\}\\
\ge&\max_{\overline{W}_l\in \widetilde{S}_l}\{\int_{x \sim D_x}(||(W_{L+1}^{(l_x)}-W_{L+1}^{(l_2)})A_{l}(x)||_{-\infty}||J_{l}(x)\overline{W}_lB_{l-1}(x)||_{\infty})^2\d x\}\\
\ge&\max_{\overline{W}_l\in \widetilde{S}_l}\{P_{x\sim D_x}(||(W_{L+1}^{(l_x)}-W_{L+1}^{(l_2)})A_{l}(x)||_{-\infty}>c_l,\ ||J_{l}(x)\overline{W}_lB_{l-1}(x)||_{\infty}\ge d_l||J_{l}(x)\overline{W}_l||_{\infty},\\ &J_{l}(x)\overline{W}_l\ne0 )
(\gamma  c_l d_{l-1})^2\}\\
\end{array}
\end{equation*}

By conditions $C_3$ and $C_2$, we have
$P_{x\sim D_x}(||(W_{L+1}^{(l_x)}-W_{L+1}^{(l_2)})A_{l}(x)||_{-\infty}>c_l,\ ||J_{l}(x)\overline{W}_lB_{l-1}(x)||_{\infty}$ $\ge d_l||J_{l}(x)\overline{W}_l||_{\infty})>\alpha_l+\beta_{l-1}-1$.
By condition $C_4$ and the principle of drawer,   there exists a $\overline{W}_i\in\widetilde{S}_i$ such that
 $P_{x\sim D_x}(J_{l}(x)\overline{W}_l\ne0\ |\ ||(W_{L+1}^{(l_x)}-W_{L+1}^{(l_2)})A_{l}(x)||_{-\infty}>c_l,\ ||J_{l}(x)\overline{W}_lB_{l-1}(x)||_{\infty}\ge d_l||J_{l}(x)\overline{W}_i||_{\infty})>\gamma$. So,  there exists a $\overline{W}_l\in\widetilde{S}_l$ such that
\begin{equation*}
\renewcommand{\arraystretch}{1.5}
\begin{array}{ll}
P_{x\sim D_x}(||(W_{L+1}^{(l_x)}-W_{L+1}^{(l_2)})A_{l}(x)||_{-\infty}>c_l, ||J_{l}(x)\overline{W}_lB_{l-1}(x)||_{\infty}\ge d_l||J_{l}(x)\overline{W}_l||_{\infty}, J_{l}(x)\overline{W}_l\ne0 )\\
>\gamma (\alpha_l+\beta_{l-1}-1).
\end{array}
\end{equation*}
Then
\begin{equation*}
\renewcommand{\arraystretch}{1.5}
\begin{array}{ll}
&\max_{\overline{W}_i\in S_i}\{\int_{x \sim D_x}||M_i(x,\overline{W}_i)||^2_2\d x\}\\
\ge&\max_{\overline{W}_i\in \widetilde{S}_i}\{P_{x\sim D_x}(||(W_{L+1}^{(l_x)}-W_{L+1}^{(l_2)})A_{i}(x)||_{-\infty}>c_i,\\ & ||J_{i}(x)\overline{W}_iB_{i-1}(x)||_{\infty}\ge d_i||J_{i}(x)\overline{W}_i||_{\infty},\ J_{i}(x)\overline{W}_i\ne0 )
(\gamma  c_i d_{i-1})^2\}\\
\ge& (\gamma  c_i d_{i-1})^2 \gamma (\alpha_i+\beta_{i-1}-1).
\end{array}
\end{equation*}

Similarly, by condition $C_3$, we have
{\small
\begin{equation*}
\renewcommand{\arraystretch}{1.5}
\begin{array}{ll}
&\max_{\overline{W}_{L+1}\in S_{L+1}}\{\int_{x \sim D_x}||M_{L+1}(x,\overline{W}_{L+1})||^2_2\d x\}\\
=&\max_{\overline{W}_{L+1}\in S_{L+1}}\{\int_{x \sim D_x}||(\overline{W}_{L+1}^{(l_x)}-\overline{W}_{L+1}^{(l_2)})B_{L}(x)||^2_2\d x\}\\
\ge&\max_{\overline{W}_{L+1}\in S_{L+1}}\{\int_{x \sim D_x}||(\overline{W}_{L+1}^{(l_x)}-\overline{W}_{L+1}^{(l_2)})B_{L}(x)||_{\infty}^2\d x\}\\
\ge&\max_{\overline{W}_{L+1}\in S_{L+1}}\{\int_{x \sim D_x}I(||(\overline{W}_{L+1}^{(l_x)}-\overline{W}_{L+1}^{(l_2)})B_{L}(x)||_{\infty}\ge d_{L}||\overline{W}_{L+1}^{(l_x)}-\overline{W}_{L+1}^{(l_2)}||_{\infty})(d_{L}||\overline{W}_{L+1}^{(l_x)}-\overline{W}_{L+1}^{(l_2)}||_{\infty})^2\d x\}\\
\ge&\beta_L(d_L\gamma )^2.
\end{array}
\end{equation*}
}
Then we have the desired lower bound
\begin{equation*}
\renewcommand{\arraystretch}{1.5}
\begin{array}{ll}
&\max_{\widetilde{\F}\in \FB}\{\int_{x\sim D_x}||\nabla(\widetilde{\F}_{l_x}(x))-\nabla(\widetilde{\F}_{l_2}(x))||^2_2\d x\}\\
\ge&\int_{x\sim D_x}||\nabla(\F_{l_x}(x))-\nabla(\F_{l_2}(x))||_2^2\d x+ \sum_{i=1}^{L+1}\max_{\overline{W}_i\in S_i}\{\int_{x\sim D_x}||M_i(\overline{W}_i)||^2_2\d x\}\\
\ge&\int_{x\sim D_x}||\nabla(\F_{l_x}(x))-\nabla(\F_{l_2}(x))||_2^2\d x+(\gamma  c)^2 \alpha_1 \gamma+\sum_{i=2}^{L}(\gamma  c_i d_{i-1})^2 \gamma (\alpha_i+\beta_{i-1}-1)+\beta_L(d_L\gamma )^2.\\
\end{array}
\end{equation*}

By condition $C_1$ and the lower bound just obtained, we have
\begin{equation*}
\renewcommand{\arraystretch}{1.5}
\begin{array}{ll}
&\min_{\widetilde{\F}\in \FB}\{\frac{\int_{x\sim D_x}||\F_{l_x}(x)-\F_{l_2}(x)||_2^2\d x}{\int_{x\sim D_x}||\widetilde{\F}_{l_x}(x)-\widetilde{\F}_{l_2}(x)||_2^2\d x}\}\\
\le&\frac{\int_{x\sim D_x}||\F_{l_x}(x)-\F_{l_2}(x)||_2^2\d x}{\int_{x\sim D_x}||\nabla(\F_{l_x}(x))-\nabla(\F_{l_2}(x))||_2^2\d x+(\gamma  c_1)^2 \alpha_1 \gamma_1+\sum_{i=2}^{L}(\gamma  c_i d_{i-1})^2 \gamma_i (\alpha_i+\beta_{i-1}-1)+\beta_L(d_L\gamma )^2}\\
=&1-\frac{(\gamma  c_1)^2 \alpha_1\gamma_1+\sum_{i=2}^{L}(\gamma  c_i d_{i-1})^2 \gamma_i (\alpha_i+\beta_{i-1}-1)+\beta_L(d_L\gamma )^2}{\int_{x\sim D_x}||\nabla(\F_{l_x}(x))-\nabla(\F_{l_2}(x))||_2^2\d x+ (\gamma  c_1)^2 \alpha_1 \gamma_1+\sum_{i=2}^{L}(\gamma  c_i d_{i-1})^2 \gamma_i (\alpha_i+\beta_{i-1}-1)+\beta_L(d_L\gamma )^2}\\
\le&1-\frac{(\gamma  c_1)^2 \alpha_1 \gamma_1+\sum_{i=2}^{L}(\gamma  c_i d_{i-1})^2 \gamma_i (\alpha_i+\beta_{i-1}-1)+\beta_L(d_L\gamma )^2}{4A+(\gamma  c_1)^2 \alpha_1 \gamma_1+\sum_{i=2}^{L}(\gamma  c_i d_{i-1})^2 \gamma_i (\alpha_i+\beta_{i-1}-1)+\beta_L(d_L\gamma )^2}.
\end{array}
\end{equation*}

The theorem is proved.
\end{proof}

\section{Conclusion}

The adversarial parameter attack for DNNs is proposed.
In the attack, the adversary makes small changes to the
parameters of a trained DNN such that the attacked DNN
will keep the accuracy of the original DNN as much as possible,
but makes the robustness as low as possible.
The goal of the attack is that the attacked DNN is
imperceptible to the user and at the same time the robustness of the DNN is broken.
The existence of adversarial parameters is proved under certain conditions
and effective adversarial parameter attack algorithms are also given.
%

In general, it is still out of reach to provide provable safety DNNs in real-world applications, and one of the ways to develop safer DNN models and training methods,
and evaluate the safety of the trained model against existing attack methods.
In other words, a DNN to be deployed is considered safe
if it is safe against existing attacks in certain sense.
From this viewpoint, it is valuable to have more attack methods.
This  is similar to the cryptanalysis~\cite{it},
where much more matured theory and attack methods are developed.

\end{document}